\documentclass[journal]{IEEEtran}
\usepackage{amsmath,amssymb,amsfonts}
\usepackage{cite}
\usepackage{algorithm}
\usepackage{algorithmic}
\usepackage{graphicx}
\usepackage{caption}
\usepackage{lipsum}
\usepackage{enumerate}
\usepackage{amsthm}
\usepackage{color}
\usepackage{threeparttable}
\usepackage{booktabs}
\usepackage{listings}
\usepackage{multirow}
\usepackage{color,soul}
\usepackage[]{hyperref}
\usepackage{todonotes}
\usepackage{bm}
\newcommand{\Exp}{\mathop{\mathbb E}\displaylimits}
\newtheorem{definition}{Definition}
\newtheorem{lemma}{Lemma}
\newtheorem{theorem}{Theorem}

\newtheorem{corollary}{Corollary}

\ifCLASSOPTIONcompsoc
  \usepackage[caption=false,font=normalsize,labelfont=sf,textfont=sf]{subfig}
\else
  \usepackage[caption=false,font=footnotesize]{subfig}
\fi
\begin{document}
\title{Smoothing Policy Iteration for Zero-sum Markov Games}

\author{Yangang Ren\textsuperscript{\#}, Yao Lyu\textsuperscript{\#}, Wenxuan Wang, Shengbo Eben Li*, Zeyang Li,  Jingliang Duan
\thanks{This work is supported by  the National Natural Science Foundation of China under No. U20A20334. It is also partially supported by Tsinghua University-Toyota Joint Research Center for AI Technology of Automated Vehicle. All correspondences should be sent to S. Eben Li with email: lisb04@gmail.com.}
\thanks{Y. Ren, Y. Lyu, W. Wang and S. Eben Li are with State Key Lab of Automotive Safety and Energy, School of Vehicle and Mobility, Tsinghua University, Beijing, China. 
{\tt\small Email: (ryg18, y-lv19, wang-wx18)@mails.tsinghua.edu.cn; lishbo@tsinghua.edu.cn}.
}
\thanks{Z. Li is with the State Key Laboratory of Tribology,  Department of Mechanical Engineering, Tsinghua University, Beijing, China. {\tt\small Email: li-zy21@mails.tsinghua.edu.cn}.
}
\thanks{J. Duan is with the School of Mechanical Engineering, University of Science and Technology Beijing, China.
{\tt\small Email:duanjl@ustb.edu.cn}.
}
}


\maketitle

\begin{abstract}
Zero-sum Markov Games (MGs) has been an efficient framework for multi-agent systems and robust control, wherein a minimax problem is constructed to solve the equilibrium policies.
At present, this formulation is well studied under tabular settings wherein the maximum operator is primarily and exactly solved to calculate the worst-case value function.
However, it is non-trivial to extend such methods to handle complex tasks, as finding the maximum over large-scale action spaces is usually cumbersome.
In this paper, we propose the smoothing policy iteration (SPI) algorithm to solve the zero-sum MGs approximately, where the maximum operator is 
replaced by the weighted LogSumExp (WLSE) function to obtain the nearly optimal equilibrium policies. 
Specially, the adversarial policy is
served as the weight function to enable an efficient sampling over action spaces.
We also prove the convergence of SPI and analyze its approximation error in $\infty -$norm based on the contraction mapping theorem.
Besides, we propose a model-based algorithm called Smooth adversarial Actor-critic (SaAC) by extending SPI with the function approximations. The target value related to WLSE function is evaluated by the sampled trajectories and then mean square error is constructed to optimize the value function, and the gradient-ascent-descent methods are adopted to optimize the protagonist and adversarial policies jointly. 
In addition, we incorporate the reparameterization technique in model-based gradient back-propagation to prevent the gradient vanishing due to sampling from the stochastic policies.
We verify our algorithm in both tabular and function approximation settings. Results show that SPI can approximate the worst-case value function with a high accuracy and SaAC can stabilize the training process and improve the adversarial robustness in a large margin.
\end{abstract}

\begin{IEEEkeywords}
Approximate dynamic programming, Adversarial robustness, Reinforcement Learning, Stochastic game.
\end{IEEEkeywords}

\section{Introduction}

Reinforcement Learning (RL) has achieved tremendous progress in several notable decision-making problems such as playing the game of Go \cite{silver2017mastering} and driving an autonomous vehicle\cite{ma2018improved, ren2022improve}. Interestingly, many of these problems can be formulated as a zero-sum Markov Games (MGs) wherein two opposing players or teams, called protagonist and adversary respectively, aim to learn jointly through competition.
Formally, zero-sum MGs is usually constructed as a minimax formulation wherein the protagonist attempts to minimize the accumulated reward while the adversary tries maximizing it.
Thus, they can be seen as the generalization of both game-theory and the classic Markov Decision Process (MDP) in traditional RL\cite{sutton2018reinforcement}.
Over the past years, zero-sum MGs has been largely studied in muti-agent systems\cite{littman1994markov, zhang2020model} or robust control\cite{morimoto2005robust,li2020reinforcement}.
As the optimal behaviour of a player must take into account the other one, Nash Equilibrium (NE) policies are always adopted to describe the optimal solutions at which both players do not have an incentive to vary their current policy even if they could. 
Additionally, the adversarial Bellman equation is usually proposed to 
ease the solution of this minimax problem, which theoretically gives a sufficient condition of NE policies\cite{patek1997stochastic}. 
Then, similar to MDP formulation, there exit two mechanisms to iteratively solve the Bellman equation: value iteration and policy iteration.

Value iteration aims to solve the adversarial Bellman equation by iteratively applying the Bellman operator at each iteration, which will produce a sequence of values and finally converge to the equilibrium value of NE policies\cite{Shapley1953}. This method is well researched now with the complete convergence proof under the tabular setting and has been extended widely with function approximations.
Notably, minimax Q-learning was firstly proposed to compute the optimal policies of both the agents utilizing the state and reward samples obtained from the environment\cite{littman1994markov}. 
After that, minimax TD-learning was proposed to utilizing the concept of TD-learning \cite{dahl2000minimax} and the successive relaxation was adopted to obtain a faster computation of equilibrium value\cite{Diddigi_2022}.
Yang \emph{et al.} further proposed minimax-DQN inspired by the huge success of DQN in Atari games \cite{mnih2015human} and also characterized its performance in terms of both the algorithmic and statistical rates of convergence\cite{fan2020theoretical}.
This method can be viewed as a combination of the minimax-Q learning for tabular zero-sum MGs and deep neural networks for function approximation. 
Meanwhile, Pan \emph{et al.} adopted the similar idea to train the autonomous vehicle in driving simulator, wherein the vehicle can execute nine discrete actions to steer itself to finish the whole race \cite{pan2019risk}. Nevertheless, value-iteration methods can only handle the simple tasks with discrete action spaces for the reason of precisely solve the minimax problem in each step.

Policy iteration aims to find the NE policies progressively by conducting policy evaluation (PEV) and policy improvement (PIM) iteratively. One straightford method is called the Naive Policy Iteration (NPI) \cite{Pollatschek1969newt}, wherein its PEV attempts to calculate a joint value function related to both protagonist and adversary. In PIM, it needs to solve a minimax problem w.r.t. the calculated value function to find another pairs of protagonist and adversary.
Although it has empirically found that NPI is quite efficient in practice, this algorithm was not proven to converge in all zero-sum MGs. In fact, it has been showcased that NPI might have cyclic behaviors and fail to converge to any NE in zero-sum MGs \cite{van1978discounted}.
Regardless of this, many works prefer to extend this method with function approximation to train the optimal policy for high-dimensional tasks. 
Morimoto \emph{et al.} firstly formulated a differential game with normalized gaussian networks, and derived the gradient-based method of the update of protagonist and adversary policy based on the linear dynamic models\cite{morimoto2005robust}.
Pinto \emph{et al.} combined this method with with deep neural network to improve the robustness of protagonist, in which the protagonist and adversary policies are trained alternatively, with one being fixed whilst the other adapts \cite{pinto2017robust}.
Vinitsky \emph{et al.} extended a population-based augmentation of the one-to-one minimax formulation in which a population of adversaries were randomly initialized to compete against the protagonist, showing that this approach results in a more robust policy \cite{vinitsky2020robust}.
Although it can be easy to implement and can handle high-dimensional and continuous action spaces, NPI-based methods lack theoretical robustness and optimality guarantee, and suffer from oscillation during training.

Another more stable policy iteration is the Asynchronous Policy Iteration (API)\cite{hoffman1066sg}, which could make sure a strictly converged protagonist policy motivated by the policy improvement theorem in MDP\cite{lse2023reinforcement}. This method is characterized in that PEV step aims to evaluate the worst-case value function of the protagonist by assuming the adversary always to behave severely. 
Afterwards, PIM attempts to solve a minimax problem w.r.t. the calculated worst-case value function to find another better protagonist policy.
Although having a complete convergence proof, this method is cumbersome in application resulting from solving an MDP problem exactly in PEV\cite{rao1973algorithms}. 
Therefore, some works like \cite{perolat15, Huang2021towards} proposed to evaluate the protagonist approximately under tabular setting and gave the corresponding error bound in the mild assumptions.
Recently, Wang \emph{et al.} develops the gradient-based method to solve the robust protagonist policy parameterized with neural networks while regarding the model uncertainty as an adversary \cite{wang2022policy}.
To provide the convergence rate, they adopted the LogSumExp (LSE) operator to approximate the maximum operator and thus provide a surrogate
Bellman operator for policy evaluation\cite{wang2021online}.
However, the current methods are only applicable for discrete action space for the sake of traversing all actions to seek the worst-case adversarial one.

To sum up, the API-based methods enjoys the convergence guarantee, but it can't be well applied to conquer large-scale tasks as the exact solution of maximum operation. Besides, the NPI-based methods have achieved empirical success in continuous action space domains meanwhile lacking any theoretical analysis to guarantee the convergence.
In this paper, we propose a novel smoothing policy iteration  (SPI) method for zero-sum MGs, which aims to obtain the nearly optimal policies and is applicable to complex tasks with continuous action space.
The contributions emphasize in three parts: 

1) SPI algorithm is proposed by approximating the maximum operator with the Weighted-LSE (WLSE) function.
Specially, the adversarial policy is served as the weight function of LSE such that we can provably sample the worst adversarial actions from the large-scale action space. And with the approximation of worst-case value function in API, SPI accesses the fairly good training stability and thus provably generates a converged protagonist policy.
Based on the contraction mapping theorem, we prove the convergence of SPI and analyze its approximation error against API in $\infty -$norm.

2) Based on the developed SPI framework, we propose a model-based RL algorithm called Smooth adversarial Actor-critic (SaAC) by introducing the deep neural networks as the function approximators.
For the critic update, the WLSE operator is calculated with the trajectories sampled by protagonist and adversarial policies, which will serve as the target value and then mean square error is constructed to optimize the value network.
For the actor update, the minimax formulation is established and the gradient-ascent-descent methods are developed to optimize the protagonist and adversarial networks respectively, where we incorporate the reparameterization technique in model-based gradient backpropgation to prevent the gradient vanishing due to sampling from the stochastic policies.

3) Extensive simulations are conducted to evaluate the proposed method from both tabular setting and function approximation setting.
The former shows that SPI can acquire a closely approximation of API and thus has a pretty good stability, while NPI suffers divergence under some initial conditions. The latter demonstrates that SaAC algorithm can be more stable to solve the large-scale and continuous action space tasks, and thus can be widely used to train adversarially the robust policy for real-world industrial applications.

\section{Preliminary}

In this section, we first introduce the basic concepts of the zero-sum MGs and the adversarial Bellman equation. Under this framework, two mainstream policy iteration methods are presented, including naive policy iteration (NPI) and asynchronous policy iteration (API). 

\subsection{Zero-sum Markov games}

Zero-sum MGs has been an efficient framework for multi-agent learning and robust learning\cite{iyengar2005robust, ren2020improving}. 
Specifically, It can be characterized by a tuple $<\mathcal{S}, \mathcal{A}, \mathcal{U}, p, r, \gamma>$, where $\mathcal{S}$, $\mathcal{A}$, and $\mathcal{U}$ indicate the state space, protagonist action and adversarial action space, respectively. $p(s'|s,a,u)$ is the transition model which determines the distribution of the next state $s'$ depending on the protagonist action $a \in \mathcal{A}$ and the adversarial action $u \in \mathcal{U}$ at the current state $s \in \mathcal{S}$.
Given the state $s_t$ at time $t$, the protagonist policy $\pi(a_t|s_t)$ represents the distribution of the action $a_t$, i.e., $a_t \sim \pi(a_t|s_t)$, while the adversarial policy $\mu(u_t|s_t)$ outputs the distribution of the adversarial action $u_t$, i.e., $u_t \sim \mu(u_t|s_t)$, and both jointly transit the state $s_t$ to the next state $s_{t+1}$, i.e., $s_{t+1} \sim p(s_{t+1}|s_t,a_t,u_t)$. During this process, the environment feeds the reward $
r(s_t,a_t,u_t)$ back to the protagonist agent, and the adversarial agent receives an opposite reward $-r(s_t,a_t,u_t)$. Moreover, $\gamma \in [0, 1)$ is the discount factor.

In this framework, the state-value function $v^{\pi,\mu}(s)$ depends on both policies $\pi$ and $\mu$, which represents the long-term accumulated discounted reward induced by these two policies, i.e.,
\begin{equation}
\label{eq.def_state_value}
\begin{aligned}
v^{\pi,\mu}(s_t)=\mathbb{E}_{\substack{\pi,\mu}}\left[\sum_{i=0}^{\infty}\gamma^{i} r(s_{t+i}, a_{t+i},u_{t+i})\right].
\end{aligned}
\end{equation}
And the optimal value function $v^*(s)$ is defined as the saddle point of the state-value function, i.e.,
\begin{equation}
\label{eq.def_opt_state_value}
\begin{aligned}
v^*(s)\overset{{\rm def}}{=}v^{\pi^*,\mu^*}(s)&=\min\limits_{\pi} \max \limits_{\mu} v^{\pi,\mu}(s)\\
&=\max \limits_{\mu}\min\limits_{\pi} v^{\pi,\mu}(s),
\end{aligned}
\end{equation}
which denotes the optimal policies satisfying the NE condition:
\begin{equation}
\label{eq.def_opt_policy}
\nonumber
\begin{aligned}
v^{\pi^*, \mu}(s) \le v^{\pi^*, \mu^*}(s) \le v^{\pi, \mu^*}(s), \forall s \in \mathcal{S}.
\end{aligned}
\end{equation}
It has shown that the optimal NE policies $\pi^*$ and $\mu^*$ always exit under the assumption of stochastic policy formulation. Furthermore, they can be solved directly from the adversarial Bellman equation \cite{patek1997stochastic}.

\begin{definition}[Adversarial Bellman equation]
In zero-sum MGs, the solution of
adversarial Bellman equation should satisfy the definition \eqref{eq.def_opt_state_value} of the optimal value function $v^{*}$, i.e.,
\begin{equation}
\label{eq.def_bellman_eq}
\begin{aligned}
v^*(s)=\min\limits_{\pi} \max \limits_{\mu}\sum_{a}\pi\sum_{u}\mu\sum_{s'}p\left[r+\gamma v^*(s')\right].
\end{aligned}
\end{equation}
\end{definition}
\noindent
That is, like classical MDP setting, adversarial Bellman equation actually provides a sufficient condition for the optimal solutions of zero-sum MGs. By solving it, we can simultaneously obtain both $v^{*}$ and $\pi^*$,$\mu^*$.

\subsection{Naive Policy Iteration}
Naive policy iteration (NPI) is the most straightforward method to solve the adversarial Bellman equation which is rather similar to the policy iteration in MDP settings \cite{sutton2018reinforcement}. In PEV, NPI aims to evaluate the joint value function defined in \eqref{eq.def_state_value} given current policies $\pi$, $\mu$.
wherein the corresponding Bellman operator $\mathcal{T}^{\pi,\mu}$ can be written as:
\begin{equation}
\label{eq.npi_self_consistency_condition}
\nonumber
    \mathcal{T}^{\pi,\mu}v^{\pi,\mu}(s)=\sum_{a}\pi\sum_{u}\mu\sum_{s'}p\left[r+\gamma v^{\pi,\mu}(s')\right].
\end{equation}
It has been proven that $\mathcal{T}^{\pi,\mu}$ is a $\gamma$-contraction mapping. Thus, we can utilize the iterative manner to find the fixed-point of this operator starting with a random initialization, i.e.,
\begin{equation}
\label{eq.npi_pev}
    V^{\pi,\mu}_{j+1}(s)=\mathcal{T}^{\pi,\mu}V^{\pi,\mu}_{j}(s)=\sum_{a}\pi\sum_{u}\mu\sum_{s'}p\left[r+\gamma V^{\pi,\mu}_{j}(s')\right],
\end{equation}
where $V^{\pi,\mu}_{j}(s)$ is the estimation of the state value $v^{\pi,\mu}(s)$ and $j$ denotes iterative step in PEV. And this kind of iteration will eventually converge to $v^{\pi,\mu}(s)$, i.e., $\lim_{j\rightarrow \infty}V^{\pi,\mu}_{j}(s)=v^{\pi,\mu}(s)$.
As for PIM, a minimax formulation will be solved to find another round of policies $\pi'$, $\mu'$ based on the calculated $v^{\pi,\mu}$, i.e.,
\begin{equation}
\label{eq.npi_pim}
\begin{aligned}
\pi', \mu' =\arg\min\limits_{\pi} \max \limits_{\mu}\sum_{a}\pi\sum_{u}\mu\sum_{s'}p\left[r+\gamma v^{\pi,\mu}(s')\right].
\end{aligned}
\end{equation}
And this actually formulates a game matrix which can be solve exactly using two dual linear programming under tabular case, then the solved policies will be evaluated by another round of PEV as in \eqref{eq.npi_pev} again.

Although NPI seems intuitive and be efficient in some applications, this method lacks of convergence guarantee, making it usually unstable and easy to diverge, especially integrated with nonlinear function approximators such as neural networks.

\subsection{Asynchronous Policy iteration}
Asynchronous policy Iteration (API) is another more stable method to 
solve the adversarial Bellman equation. Specially, it will evaluate the worst-case value function $v^{\pi}$, which can be regarded as the best response  of the adversarial agent to the protagonist policy $\pi$. This worst-case Bellman operator $\mathcal{T}^{\pi}$ can be defined as:
\begin{equation}
\label{eq.api_self_consistency_condition}
\begin{aligned}
v^{\pi}(s)=\mathcal{T}^{\pi}v^{\pi}(s)=\max\limits_{u\in\mathcal{U}}\sum_a \pi \sum_{s'}p\left[r+\gamma v^{\pi}(s')\right],
\end{aligned}
\end{equation}    
In addition, $\mathcal{T}^{\pi}$ is also a $\gamma$-contraction mapping which has a unique fixed point. 
Besides, we can attain this fixed point iteratively
using
\begin{equation}
\label{eq.api_pev}
\begin{aligned}
V^{\pi}_{j+1}(s)=\mathcal{T}^{\pi}V^{\pi}_{j}(s)=\max\limits_{u\in\mathcal{U}}\sum_a \pi \sum_{s'}p\left[r+\gamma V^{\pi}_{j}(s')\right].
\end{aligned}
\end{equation}
The PIM of API also attempts to solve the minimax formulation based on $v^{\pi}$
\begin{equation}
\label{eq.api_pim}
    \pi', \mu'=\arg\min\limits_{\pi} \max \limits_{\mu}\sum_{a}\pi\sum_{u}\mu\sum_{s'}p\left[r+\gamma v^{\pi}(s')\right].
\end{equation}
Different with NPI, here $\mu’$ will not 
participate in the next round of PEV wherein only the worst adversarial action will be used to evaluate $\pi’$ in each iteration step.

The above procedure will generate a monotonically decreasing sequence of value functions, i.e.,
$v^{\pi'}(s) \le v^{\pi}(s), \forall s \in \mathcal{S}$, that is, $\pi'$ is a better policy than $\pi$. 
Therefore, $v^{\pi'}$ will strictly converge to $v^*$, i.e., $\lim_{k \to \infty}v^{\pi_k}=v^*$, where $k$ represents the iteration round of PEV and PIM.
However, the bottleneck of API lies in that the maximum operator involved in (\ref{eq.api_pev}) requires solving a nonlinear equation at each step of PEV, making API hard to be adopted to problems with continuous adversarial action space.

\section{Algorithm}
In this section, we propose to approximate the maximum operator with the weight-LSE function and thus design the smoothing policy iteration (SPI). 
In addition, we combine SPI with parameterized functions
and develop the gradient-based rules for their update, resulting in the smoothing adversarial actor-critic (SaAC) to deal with complex tasks with high-dimensional and continuous state and action spaces.

\subsection{Smoothing Policy Iteration}
Intuitively, API methods enjoy the better convergence as the exact solution of maximum operator, whereas this is non-trivial to implement in large-scale problems. On
the other side, NPI methods have the better efficiency when converging since both policies are used to sample actions to evaluate the value function.
Therefore, we aim to improve the training stability by approximating the maximum operator meanwhile using the protagonist and adversarial policies simultaneously to sample from their arbitrary action spaces. To this end, we firstly introduce the Weight LogSumExp (WLSE) function, which can be seen as an extended form of LogSumExp \cite{wang2021online} with a probability density function as the weights.
\begin{definition}[Weight LogSumExp function]
Suppose $X\in \mathbb{R}^n=[x_1,x_2,\cdots,x_n]^\top$ and a positive real value $\rho$, we define the weight LogSumExp function as 
\begin{equation}
\label{eq.def_WLSE}
    {\rm WLSE}(X)=\frac{1}{\rho}\log \sum_{i=1}^{n}{w_i\exp(\rho x_i)}
\end{equation}
where $\rho>0$ is the approximation factor, $w_i$ is the weight of $x_i$ satisfying
$\sum_{i=1}^{n}{w_i}=1$ and $w_i \ge 0, \forall{i} \in \{1,2,\cdots,n\}$.
\end{definition}

Then, we can see that this function can closely approach to the maximum item of $X$, i.e., $x_m=\max\limits_{i} {X}$, controlled by $\rho$ and the weight of $x_m$.
\begin{lemma}[Error bound of WLSE function]
\label{lemma_Error bound of WLSE operator}
Suppose $x_m=\max\limits_{i} {X}$ and $w_m$ is the weight of $x_m$, the approximate error of WLSE function to the maximum value is
\begin{equation}
|\operatorname{WLSE}(X)-x_m| \leqslant\left|\frac{\log w_m}{\rho}\right|.
\end{equation}
\begin{proof}
Given $\rho>0$, $x_m=\max\{x_1,x_2,\cdots,x_n\} $ and $w_m \in [0, 1]$ is the corresponding weight, we have
$$w_m \exp \left(\rho x_m\right) \leqslant \sum_{i=1}^n w_i \exp \left(\rho x_i\right) \leqslant \exp \left(\rho x_m\right).$$
Take the logarithm of the above equation and rearrange it, we have
$$
\frac{\log w_m}{\rho} \leqslant \frac{1}{\rho}\log \sum_{i=1}^n w_i \exp \left(\rho x_i\right) - x_m \leqslant 0.
$$
It is obvious that WLSE always makes a lower approximaton to the orginal maximum, and then we take the absolute value of both sides, i.e.,
$$ \left|\frac{1}{\rho}\log \sum_{i=1}^n w_i \exp \left(\rho x_i\right) - x_m \right| \leqslant \left|\frac{\log w_m}{\rho}\right|.$$
\end{proof}
\end{lemma}
Obviously, $\rho$ and $w_m$ determines the approximate error bound of WLSE function respectively,
which will approach to zero when $\rho \to \infty$ for any $w_m>0$ and a bigger $w_m$ will also lead to a tighter error bound.
\begin{corollary}
The WLSE function is equivalent to the maximum operator as $\rho$ goes to infinity, i.e.,
\begin{equation}
\nonumber
\lim_{\rho \to \infty}{\rm WLSE}(X) = \max X.
\end{equation}
\begin{proof}
If $w_m=1$, WLSE will be deduced samely as $x_m$ for any $\rho > 0$, i.e., ${\rm WLSE}(X) = \max X$, which accords with the corollary condition. Otherwise, since $\log w_m$ is finite as $w_m \in (0,1)$, we have
$$\lim_{\rho \to \infty}|\operatorname{WLSE}(X)-\max X| \leqslant \lim_{\rho \to \infty}\left|\frac{\log w_m}{\rho}\right|=0.$$
\end{proof}
\end{corollary}

Based on this, we can approximate the maximum operator in (\ref{eq.api_self_consistency_condition}) with the WLSE function where the adversarial policy $\mu$ can be served as the weight function of different trajectories sampled by the protagonist policy $\pi$.
The smoothing worst-case Bellman operator $\widetilde{\mathcal{T}}^{\pi}$ can be defined as
\begin{equation}
\label{eq.spi_self_consistency_condition}
\begin{aligned}
\widetilde{\mathcal{T}}^{\pi}v^{\pi}_\rho(s)
=\frac{1}{\rho}{\log \sum_{u}\mu \exp\bigg\{\rho\sum_{a}\pi\sum_{s'}p\big[r+\gamma v^{\pi}_\rho(s')\big]\bigg\}},
\end{aligned}
\end{equation}
where $v^\pi_\rho(s)$ is the state-value function calculated by WLSE function with some $\rho > 0$, and $\widetilde{\mathcal{T}}^{\pi}$ is the approximation of $\mathcal{T}^{\pi}$. 
Under this scheme, we primarily prove that $\widetilde{\mathcal{T}}^{\pi}$ is also a $\gamma$-contraction mapping such that the iterative manner can be used to find its fixed value in PEV.
\begin{lemma}[1-Lipschitz property of WLSE]
\label{lemma_1-Lipschitz property of WLSE operator}
The {\rm WLSE} function is 1-Lipschitz. That is 
\begin{equation}
\left|{\rm WLSE}(X)-{\rm WLSE}(Y) \right| \le {\lVert X-Y \rVert}_{\infty}, \forall{X,Y \in \mathbb{R}^n}.
\end{equation}
\begin{proof}
By the differential mean value theorem, $\forall{X,Y \in \mathbb{R}^n}$, $\exists h\in(0,1)$ such that
$$\operatorname{WLSE}(\mathrm{X})-\mathrm{WLSE}(\mathrm{Y})=\nabla \operatorname{WLSE}[X+h(X-Y)](X-Y).$$
Noticing that $$\frac{\partial{\rm WLSE(X)}}{\partial x_i}=\frac{w_i\exp(\rho x_i)}{\sum_{i=0}^{n}{w_i\exp(\rho x_i)}},$$
we can obtain $$\|\nabla{\rm WLSE(X)} \|_{\infty} \le 1. $$
According to Cauchy-Schwarz inequality, we have
$$
\begin{aligned}
&|\mathrm{WLSE}(X)-\mathrm{WLSE}(Y)| \\
& \leqslant\|\nabla\operatorname{WLSE}[Y+h(X-Y)]\|_{\infty} \left\| X-Y \right\|_{\infty} \\
& \leqslant\|X-Y\|_{\infty}.
\end{aligned}
$$
\end{proof}
\end{lemma}
\noindent
Then under Lemma \ref{lemma_1-Lipschitz property of WLSE operator}, we show that 
$\widetilde{\mathcal{T}}^{\pi}$ is a $\gamma$-contraction.
\begin{theorem}[$\gamma$-contraction property of $\widetilde{\mathcal{T}}^{\pi}$]
\label{theorem_gamma_contraction_of_smoothing_operator}
$\widetilde{\mathcal{T}}^{\pi}$ is a $\gamma$-contraction mapping, i.e.,
\begin{equation}
{\lVert\widetilde{\mathcal{T}}^{\pi}V^{\pi}_{j+1}-\widetilde{\mathcal{T}}^{\pi}V^{\pi}_{j} \rVert}_{\infty} \le \gamma {\lVert V^{\pi}_{j+1}-V^{\pi}_{j} \rVert}_{\infty}.
\end{equation}
\end{theorem}

\begin{proof}
According to Lemma \ref{lemma_1-Lipschitz property of WLSE operator}, we can derive that
$$
\begin{aligned}
&{\left|\widetilde{\mathcal{T}}V^{\pi}_{j+1}(s)-\widetilde{\mathcal{T}}V^{\pi}_{j}(s) \right|} \\
& =\Bigg| \frac{1}{\rho}{\log \sum_{u}\mu \exp\bigg\{\rho\sum_{a}\pi\sum_{s'}p\big[r+\gamma V^{\pi}_{j+1}(s')\big]\bigg\}} - \\
&\qquad \frac{1}{\rho}{\log \sum_{u}\mu \exp\bigg\{\rho\sum_{a}\pi\sum_{s'}p\big[r+\gamma V^{\pi}_{j}(s')\big]\bigg\}}
\Bigg| \\
& \le \max_u \Big| \sum_{a}\pi\sum_{s'}p\big[r+\gamma V^{\pi}_{j+1}(s')\big] - \\
&\qquad \qquad \sum_{a}\pi\sum_{s'}p\big[r+\gamma V^{\pi}_{j}(s')\big] \Big| \\
&=\gamma \max_u \Big| \sum_{a}\pi\sum_{s'}p \big[V^{\pi}_{j+1}(s') -V^{\pi}_{j}(s')\big]\Big| \\
&\le \gamma \max_u {\Big| \sum_{a}\pi\sum_{s'}p 
\max_{s \in \mathcal{S}}\big[V^{\pi}_{j+1}(s) -V^{\pi}_{j}(s)\big]\Big|} \\
&=\gamma {\left \| V^{\pi}_{j+1}-V^{\pi}_{j} \right \|}_{\infty}.
\end{aligned}
$$
Since the above inequality holds for any state $s \in \mathcal{S}$, we finally have
$${\lVert\widetilde{\mathcal{T}}^{\pi}V^{\pi}_{j+1}-\widetilde{\mathcal{T}}^{\pi}V^{\pi}_{j} \rVert}_{\infty} \le \gamma {\lVert V^{\pi}_{j+1}-V^{\pi}_{j} \rVert}_{\infty}.$$
\end{proof}
Up to now, we can iteratively solve the fixed point of $\widetilde{\mathcal{T}}^{\pi}$, denoted as $v^\pi_\rho$, to conduct the policy evaluation, i.e.,
\begin{equation}
\label{eq.spi_pev}
\begin{aligned}
V^{\pi}_{j+1}(s)&=\widetilde{\mathcal{T}}^{\pi}V^{\pi}_{j}(s)\\
&=\frac{1}{\rho}{\log \sum_{u}\mu \exp\bigg\{\rho\sum_{a}\pi\sum_{s'}p\big[r+\gamma V^{\pi}_{j}(s')\big]\bigg\}},
\end{aligned}
\end{equation}
wherein $j$ is the iterative step and $v^{\pi}_\rho(s)=\lim_{j \to \infty}V_{j}^{\pi}(s)$. Furthermore, the approximation error of $v^\pi_\rho$ to the true value function $v^{\pi}$ can be bounded by $\rho$ and maximum value of $\mu$.
\begin{theorem}[Error bound of PEV]
\label{theorem_Error bound of PEV}
Suppose $v^{\pi}$, $v_{\rho}^{\pi}$ are the fixed points of $\mathcal{T}^{\pi}$ and $\widetilde{\mathcal{T}}^{\pi}$ respectively, i.e., $v^{\pi}=\mathcal{T}^{\pi}v^{\pi}$, $v_{\rho}^{\pi}=\widetilde{\mathcal{T}}^{\pi}v_{\rho}^{\pi}$, the approximate error is
\begin{equation}
\label{eq.v_bound}
\begin{aligned}
{\left \| v^{\pi}-v_{\rho}^{\pi} \right \|}_{\infty} \le \frac{1}{1-\gamma}\left\|\frac{\log \mu_m}{\rho}\right\|_\infty
\end{aligned}
\end{equation}
where $\mu_m$ is the maximum value of the adversarial policy $\mu$.
\end{theorem}
\begin{proof}
According to Lemma \ref{lemma_Error bound of WLSE operator}, Theorem \ref{theorem_gamma_contraction_of_smoothing_operator} and the triangle inequality, we can derive that

$$
\begin{aligned}
&{\left\| v^{\pi}-v_{\rho}^{\pi} \right\|}_{\infty} \\
&= {\left\| \mathcal{T}^{\pi}v^{\pi}- \widetilde{\mathcal{T}}^{\pi}v_{\rho}^{\pi}\right\|}_\infty \\
&\le {\left\| \mathcal{T}^{\pi}v^{\pi}- \mathcal{T}^{\pi}v_{\rho}^{\pi}\right \|}_\infty + {\left\| \mathcal{T}^{\pi}v_{\rho}^{\pi}- \widetilde{\mathcal{T}}^{\pi}v_{\rho}^{\pi}\right\|}_\infty \\
&\le \gamma{\left \| v^{\pi}-v_{\rho}^{\pi} \right \|}_{\infty} + \max_s \left|\frac{\log \max_u \mu(\cdot|s)}{\rho}\right| \\
&= \gamma{\left \| v^{\pi}-v_{\rho}^{\pi} \right \|}_{\infty} + \left\|\frac{\log \mu_m}{\rho}\right\|_\infty.
\end{aligned}
$$
Rearrange the above inequality we immediately conclude \eqref{eq.v_bound}.
\end{proof}

Once the smooth value function $v^\pi_\rho$ is acquired in PEV, the following PIM aims to find another round of better policies $\pi'$ and $\mu'$ by solving the following minimax problem
\begin{equation}
\label{eq.spi_pim}
\begin{aligned}
\pi', \mu'=\arg\min\limits_{\pi} \max \limits_{\mu} \sum_{a}\pi\sum_{u}\mu\sum_{s'}p\big[r+\gamma v^\pi_\rho(s)\big].
\end{aligned}
\end{equation}
Thus, PEV in \eqref{eq.spi_pev} and PIM in \eqref{eq.spi_pim} constitutes the smoothing policy iteration (SPI), wherein these two phases iterate to obtain the optimal value function $v_\rho^{*}$ and optimal policise $\pi^*$, $\mu^*$. 
As the introducing of the approximation operator, SPI actually acquires the nearly optimal NE solution of the adversarial bellman equation in \eqref{eq.def_bellman_eq}, wherein the error bounds between $v_\rho^{*}$ and $v^{*}$optimal solutions can be analysed explicitly.

\begin{lemma}[Error bound of Bellman equation\cite{patek1997stochastic}]
\label{lemma_error_bound_optimal_value}
Suppose $\epsilon$ is the approximate error bound of PEV and $\widetilde{v}^{*}$ is the corresponding nearly-optimal value function, the approximate error bound of $\widetilde{v}^{*}$ to $v^{*}$is
\begin{equation}
\begin{aligned}
{\left \| \widetilde{v}-v^{*} \right \|}_{\infty}
\le \frac{2\gamma}{(1-\gamma)^2} \epsilon,
\end{aligned}
\end{equation}
\end{lemma}
\noindent Therefore, we characterize the optimality of SPI by considering the error bound in Theorem \ref{theorem_Error bound of PEV}:
\begin{corollary}
Suppose $v_{\rho}^{*}$ is the optimal value function outputed by SPI and $v^{*}$ is the true optimal value function, the approximate error bound of $v_{\rho}^{*}$ to $v^{*}$ is
\begin{equation}
\label{eq.spi_optimal_value_error_bound}
\begin{aligned}
{\left \| v_{\rho}^{*}-v^{*} \right \|}_{\infty}
\le \frac{2\gamma}{(1-\gamma)^3} \left\|\frac{\log \mu_m}{\rho}\right\|_\infty.
\end{aligned}
\end{equation}
\end{corollary}
\noindent It is obvious that SPI converges to the NE solution as $\rho \to \infty$.

\subsection{The influence of weight function}
Here, we analyze the influence of weight function choice to the approximation error under tabular case.
Specially, we choose the adversarial policy $\mu$ as the weight function in \eqref{eq.spi_self_consistency_condition}, which will provably lead to a more accurate approximation to the maximum operation.
To that end, we firstly show that $\widetilde{\mathcal{T}}^{\pi}$ is a monotonous operator.
\begin{lemma}[Monotonicity of $\widetilde{\mathcal{T}}^{\pi}$]
\label{lemma_error_monoto}
Let $V^{\pi}_j, W^{\pi}_j \in \mathbb{R}$ such that $V^{\pi}_j(s) \ge W^{\pi}_j(s)$ holds for $s \in \mathcal{S}$ and 
$V^{\pi}_{j+1}(s)=\widetilde{\mathcal{T}}^{\pi}V^{\pi}_j(s'), 
W^{\pi}_{j+1}(s)=\widetilde{\mathcal{T}}^{\pi}W^{\pi}_j(s')$. Then,
$$
V^{\pi}_{j+1}(s) \ge W^{\pi}_{j+1}(s), \forall s \in \mathcal{S}
$$
\end{lemma}
\begin{proof}
As $V^{\pi}_j(s') \ge W^{\pi}_j(s')$, we can conclude for any $u \sim \mu$
\begin{equation}
\nonumber
\begin{aligned}
&\mu(u) \exp\bigg\{\rho\sum_{a}\pi\sum_{s'}p\big[r+\gamma V^{\pi}_j(s')\big]\bigg\}\\
&\ge \mu(u) \exp\bigg\{\rho\sum_{a}\pi\sum_{s'}p\big[r+\gamma W^{\pi}_j(s')\big]\bigg\} \ge  0
\end{aligned}
\end{equation}
Then,
\begin{equation}
\nonumber
\begin{aligned}
&V^{\pi}_{j+1}(s)-W^{\pi}_{j+1}(s)=\widetilde{\mathcal{T}}^{\pi}V^{\pi}_j(s) -\widetilde{\mathcal{T}}^{\pi}W^{\pi}_j(s) \\
&=\frac{1}{\rho} \log \frac{\sum_{u}\mu \exp\bigg\{\rho\sum_{a}\pi\sum_{s'}p\big[r+\gamma V^{\pi}_j(s')\big]\bigg\}}{\mu(u) \exp\bigg\{\rho\sum_{a}\pi\sum_{s'}p\big[r+\gamma V^{\pi}_j(s')\big]\bigg\}} \\
&\ge 0
\end{aligned}
\end{equation}
\end{proof}
\noindent
Accordingly, as WLSE function provides a lower approximation to maximum as shown in Lemma \ref{lemma_Error bound of WLSE operator}, a good initialization $V^{\pi}_0$ will generate a sequence of $\{V^{\pi}_j\}$ whose fixed point $v^{\pi}_{\rho}$ will be closer to $v^{\pi}$.

On the other side, we can see that $\mu$ is always updated by the minimax formulation in \eqref{eq.spi_pim}, wherein two dual linear programs will give the solutions under tabular case. Denoting the NE value of \eqref{eq.spi_pim} as $v^{\pi',\mu'}_{\rho}$, we have
\begin{equation}
\nonumber
\begin{aligned}
&\min_{\pi',\mu'} \quad v^{\pi',\mu'} \\
\textup{s.t.}\quad &\sum_{u}\mu'\sum_{s'}p \big[r+\gamma v^\pi_\rho(s')\big]
\geq v^{\pi',\mu'},  \forall{a \in \mathcal{A}}\\
&\sum_{u \in \mathcal{U}}\mu'(u)  =  1 \\
&\mu'(u) \geq  0,   \forall{u \in \mathcal{U}}
\end{aligned}
\end{equation}
Solving this problem will generate $v^{\pi',\mu'}$ and $\mu'$. And its dual linear programming can be derived as
\begin{equation}
\nonumber
\begin{aligned}
&\max_{\pi',\mu'} \quad v^{\pi',\mu'} \\
\textup{s.t.}\quad &\sum_{a}\pi'\sum_{s'}p \big[r+\gamma v^\pi_\rho(s')\big]
\leq v^{\pi',\mu'},  \forall{u \in \mathcal{U}}\\
&\sum_{a \in \mathcal{A}}\pi'(a)  =  1 \\
&\pi'(a) \geq 0,   \forall{a \in \mathcal{A}}
\end{aligned}
\end{equation}
Similarly, we can attain $\mu'$ from this programming. According to the theorem of complementary slackness of dual problems \cite{gill2021numerical}, we have
\begin{equation}
\label{eq.comp_slack}
\begin{aligned}
\mu'(u)\big[\sum_{a}\pi'\sum_{s'}p \big[r+\gamma v^\pi_\rho(s')\big]
- v^{\pi',\mu'}\big] = 0, \forall{u \in \mathcal{U}}
\end{aligned}
\end{equation}
We can see if some $u$ makes $\sum_{a}\pi\sum_{s'}p \big[r+\gamma v^\pi_\rho(s')\big]
< v^{\pi',\mu'}$, then $\mu'(u)=0$ always holds. This means that the adversary policy only assigns non-zero probability for such adversarial actions satisfying $v^{\pi',\mu'}(s)=\sum_{a}\pi'\sum_{s'}p \big[r+\gamma v^\pi_\rho(s')\big]$.

Additionally, as $v^{\pi',\mu'}$ is the NE point, any other policy $\mu^{\#}$ will correspond to a smaller value function:
\begin{equation}
\nonumber
\begin{aligned}
v^{\pi',\mu'}&=\sum_{a}\pi'\sum_{u}\mu'\sum_{s'}p\big[r+\gamma v^\pi_\rho('s)\big] \\
&\ge \sum_{a}\pi'\sum_{u}\mu^{\#}\sum_{s'}p\big[r+\gamma v^\pi_\rho(s')\big]
\end{aligned}
\end{equation}
Specially, we assume $\mu^{\#}$ as the deterministic policy in which it only takes the adversarial action resulting in the biggest value function, i.e., $\mu^{\#}(u_m)=1.0$ where $u_m$ satisfies
\begin{equation}
\nonumber
\begin{aligned}
u_m = \arg \max\limits_{u \in \mathcal{U}} \bigg \{\sum_{a}\pi'\sum_{s'}p\big[r+\gamma v_{\rho}^{\pi}(s')\big]\bigg\}.
\end{aligned}
\end{equation}
And we can get 
$$v^{\pi',\mu'} \ge \max\limits_{u \in \mathcal{U}} \bigg \{\sum_{a}\pi'\sum_{s'}p\big[r+\gamma v_{\rho}^{\pi}(s')\big]\bigg\}.$$
Meanwhile, as $v^{\pi',\mu'}$ is the weighted average value w.r.t. the current $\pi'$, that is,
$$v^{\pi',\mu'} \le \max\limits_{u \in \mathcal{U}} \bigg \{\sum_{a}\pi'\sum_{s'}p\big[r+\gamma v_{\rho}^{\pi}(s')\big]\bigg\}.$$
Thus, it always holds that
\begin{equation}
\label{eq.max_ne}
\begin{aligned}
v^{\pi',\mu'} = \max\limits_{u \in \mathcal{U}} \bigg \{\sum_{a}\pi'\sum_{s'}p\big[r+\gamma v_{\rho}^{\pi}(s')\big]\bigg\}
\end{aligned}
\end{equation}
Combining \eqref{eq.comp_slack} and \eqref{eq.max_ne}, we conclude that choosing $\mu$ as the weight function will assign high probability over the set of the worst adversarial actions which is closer to maximum operation and thus provide a fair initialization for the next round of PEV. According to Lemma \ref{lemma_error_monoto}, this will eventually lead to a more accurate approximation and thus provably enhance the convergence speed as demonstrated in \ref{sec:simu_two-state MGs}.

\subsection{Smoothing Adversarial Actor-Critic (SaAC)}
In order to deal with complex tasks with continuous state and action spaces, we further propose the smoothing adversarial actor-critic (SaAC) algorithm based on the developed SPI method. We will consider a parameterized value function $V_{\psi}(s)$, a parameterized protagonist policy $\pi_{\theta}(\cdot|s)$ and adversarial policy $\mu_{\phi}(\cdot|s)$, wherein $\psi$, $\theta$ and $\phi$ are their trainable parameters respectively.
Based on PEV step in \eqref{eq.spi_self_consistency_condition},
$V_{\psi}(s)$ can be directly trained by minimizing the mean square error between the output of value network and the corresponding target value. Therefore, the objective of $V_{\psi}(s)$ can be written as
\begin{equation}
\label{eq.obj_saac_value}
\begin{aligned}
\min_{\psi} J_{V}(\psi)
=\frac{1}{2}\mathbb{E}_{s}\big[y_{\overline{\psi}}(s)-V_\psi(s)\big]^2, \\
\end{aligned}
\end{equation}
where $y_{\overline{\psi}}(s)$ represents the target state value of $s$ and $\overline{\psi}$ is the parameters of the target network. Concretely, it can be calculated by using the smoothing bellman operator in \eqref{eq.spi_self_consistency_condition}
\begin{equation}
\label{eq.target_value}
\begin{aligned}
y_{\overline{w}}(s)=\frac{1}{\rho}{\log \sum_{u}\mu \exp\bigg\{\rho\sum_{a}\pi\sum_{s'}p\big[r+\gamma V_{\overline{w}}(s')\big]\bigg\}}
\end{aligned}
\end{equation}
Obviously, this target value makes a important role for stable training of value update, which can be estimated by the samples generated by environment model $p$, and the protagonist policy $\pi_\theta$ and adversarial policy $\mu_\phi$, as shown in Algorithm \ref{alg:target_value}.
To that end, we sample a bunch of samples for current initial state corresponding to different adversarial actions and then the target is calculated by \eqref{eq.target_value} to estimate the expectation.

Then, the parameters of $V_{\psi}$ are updated by
\begin{equation}
\label{eq.saac_value_update}
\begin{aligned}
\psi \leftarrow \psi - \beta_{\psi} \nabla_{\psi} J_{V}(\psi),
\end{aligned}
\end{equation}
where $\beta_{\psi}$ is the learning rate of value function and $\nabla_{\psi} J_{V}({\psi})$ is the gradients of $J_V({\psi})$ w.r.t. ${\psi}$, i.e.,
$$\nabla_{\psi} J_{V}({\psi})
=-\mathbb{E}_{s}\bigg[\big(y_{\overline{\psi}}(s)-V_{\psi}(s)\big)\nabla_{\psi}V_{\psi}(s)\bigg].$$
In addition, the target network mentioned above adopts a slow-moving update rule to stabilize the learning process as
\begin{equation}
\nonumber
\begin{aligned}
\overline{\psi} \leftarrow \tau {\psi} + (1-\tau)\overline{\psi},
\end{aligned}
\end{equation}
where $\tau$ is the temperature to adjust the updating speed.

\begin{algorithm}[!htb]
\renewcommand{\algorithmicrequire}{\textbf{Input:}}
\renewcommand{\algorithmicensure}{\textbf{Output:}}
\caption{Target State Value}
\label{alg:target_value}
\begin{algorithmic}
\REQUIRE policies $\pi_{\theta}$, $\mu_{\phi}$, environment model $p$, Buffer $\mathcal{B}$, target value network $V_{\overline{\psi}}$
\STATE Initialize the initial state from $\mathcal{B}$, $s \sim \mathcal{B}$
\STATE Initialize number of trajectories $K$
\STATE Initialize value list $\mathbb{L}$
\STATE
    \FOR{each sample in $K$}
        \STATE Obtain $u \sim \mu_\phi(s)$, $a \sim \pi_\theta(s)$
        \STATE Calculate the next state with environment model $s'\sim p$
        \STATE Calculate the target value $y_i = r+\gamma V_{\overline{\psi}}(s')$ for this sample $<s, a, u, s', r>$
        \STATE Add $y_i$ into the value list $\mathbb{L}$
    \ENDFOR
\STATE
\STATE Calculate $y_{\overline{\psi}}$ using the WLSE function over $\mathbb{L}$\\
$$y_{\overline{\psi}}(s)=\frac{1}{\rho}{\log \frac{1}{K} \sum_{i=1}^{K}\exp\big\{\rho\big(r+\gamma V_{\overline{\psi}}(s')\big)\big\}}$$
\ENSURE $y_{\overline{\psi}}$
\end{algorithmic}
\end{algorithm}

As for policy update, $\pi_{\theta}$ and $\mu_{\phi}$ formalize a minimax problem based on \eqref{eq.spi_pim}, i.e.,
\begin{equation}
\label{eq.obj_pro_policy}
\nonumber
\begin{aligned}
\min\limits_{\theta} \max \limits_{\phi} J_{\pi,\mu}(\theta,\phi)=\mathbb{E}_{s}\left[r+\gamma V_{\psi}(s')\right].
\end{aligned}
\end{equation}
For this problem, the gradient-descent-ascent methods are widely adopted to update the networks simultaneously, 
\begin{equation}
\label{eq.saac_policy_upate}
\begin{aligned}
\theta &\leftarrow \theta - \beta_\theta \partial_{\theta}J_{\pi,\mu}(\theta, \phi), \\
\phi &\leftarrow \phi + \beta_\phi \partial_{\phi}J_{\pi,\mu}(\theta, \phi),
\end{aligned}
\end{equation}
where $\beta_\theta$ and $\beta_\phi$ are the learning rates of protagonist policy and adversarial policy, respectively. 
Here, we can further derive the gradients $\partial_{\theta}J_{\pi,\mu}$ and $\partial_{\phi}J_{\pi,\mu}$ in SaAC within the model-based framework,
\begin{equation}
\label{eq.protag_grad}
\begin{aligned}
\partial_{\theta}J_{\pi,\mu}&= \partial_{\theta}\mathbb{E}_{s}\left[r+\gamma V_{\psi}(s')\right] \\
&=\mathbb{E}_{s}\left[\frac{\partial{a}^{\rm T}}{\partial{\theta}}\frac{\partial{r}}{\partial{a}}+\gamma \frac{\partial{a}^{\rm T}}{\partial{\theta}}\frac{\partial{s'}^{\rm T}}{\partial{a}}\frac{\partial{V_{\psi}}}{\partial{s'}}\right] \\
&=\mathbb{E}_{s}\left[\frac{\partial{\pi_{\theta}}^{\rm T}}{\partial{\theta}}\big(\frac{\partial{r}}{\partial{a}}+\gamma \frac{\partial{p}^{\rm T}}{\partial{a}}\frac{\partial{V_{\psi}}}{\partial{s'}}\big)\right]
\end{aligned}
\end{equation}
Note that the second equation comes from the fact that $s' \sim p(s,a,u)$ and $a \sim \pi_\theta$. Besides, $\frac{\partial{r}}{\partial{a}}$ and $\frac{\partial{p}^{\rm T}}{\partial{a}}$ contribute to the gradient w.r.t. the parameters of protagonist policy, and their usage can effectively reduce the computational complexity to estimate policy gradients.
Similarly, the gradient of $\mu_\phi$ can be derived as
\begin{equation}
\label{eq.adv_grad}
\begin{aligned}
\partial_{\phi}J_{\pi,\mu}=\mathbb{E}_{s}\left[\frac{\partial{\mu_{\phi}}^{\rm T}}{\partial{\phi}}\big(\frac{\partial{r}}{\partial{u}}+\gamma \frac{\partial{p}^{\rm T}}{\partial{u}}\frac{\partial{V_{\psi}}}{\partial{s'}}\big)\right]
\end{aligned}
\end{equation}

However, the issue in calculating \eqref{eq.protag_grad} and \eqref{eq.adv_grad} is that we can not directly obtain $\frac{\partial{\pi_{\theta}}^{\rm T}}{\partial{\theta}}$ and $\frac{\partial{\mu_{\phi}}^{\rm T}}{\partial{\phi}}$in zero-sum MGs with the stochastic policy $\pi_{\theta}$ and $\mu_{\phi}$. In this condition, the sample operation of actions from a distribution will lead to the gradient fracture and thus the reparameterization trick is proposed \cite{duan2022dsac} to derive the policy gradient.
Assuming $$a=f_{\theta}(\xi;s), u=h_{\phi}(\eta;s)$$ 
where $f_{\theta}$ and $h_{\phi}$ are the reparameterized policies, $\xi$ and $\eta$ are auxiliary variables which are sampled from some fixed distribution, we can derive the executable gradient of the protagonist policy,
\begin{equation}
\label{eq.protag_repa_grad}
\begin{aligned}
\partial_{\theta}J_{\pi,\mu}=
\Exp_{\substack{s,\xi}}\left[\frac{\partial{f_{\theta}(\xi;s)}^{\rm T}}{\partial{\theta}}\big(\frac{\partial{r}}{\partial{f_{\theta}(\xi;s)}}+\gamma \frac{\partial{p}^{\rm T}}{\partial{f_{\theta}(\xi;s)}}\frac{\partial{V_{\psi}}}{\partial{s'}}\big)\right]
\end{aligned}
\end{equation}
Similarly, the gradient $\partial_{\phi}J_{\pi,\mu}$ can be derived with the reparameterization and model-based chain rules:
\begin{equation}
\label{eq.adv_repa_grad}
\begin{aligned}
\partial_{\phi}J_{\pi,\mu}=
\Exp_{\substack{s,\eta}}\left[\frac{\partial{h_{\phi}(\eta;s)}^{\rm T}}{\partial{\theta}}\big(\frac{\partial{r}}{\partial{h_{\phi}(\eta;s)}}+\gamma \frac{\partial{p}^{\rm T}}{\partial{h_{\phi}(\eta;s)}}\frac{\partial{V_{\psi}}}{\partial{s'}}\big)\right]
\end{aligned}
\end{equation}
Finally, we can conclude the details of SaAC in Algorithm \ref{alg:saac} and show the diagram in Fig.~\ref{fig.saac_diag}.
\begin{figure}[!htbp]
\centering
\captionsetup[subfigure]{justification=centering}
\includegraphics[width=1.0\linewidth]{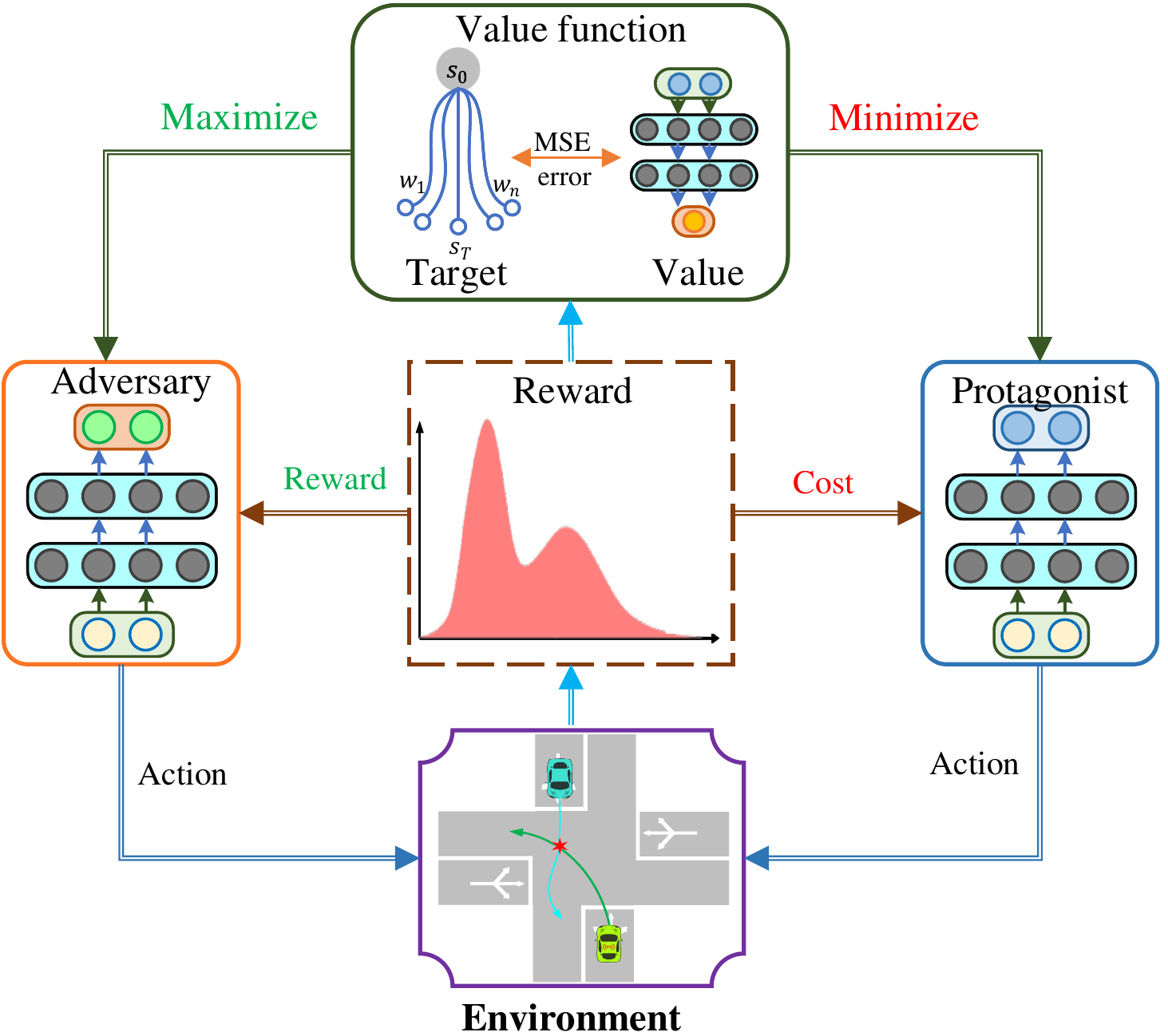}
\caption{SaAC diagram. Value function, protagonist policy and adversarial policy are carried by neural networks respectively. SaAC first updates the smoothing value based on the a bunch of samples collected by both policies. Then, the output of the value network is used to guide the update of policies by gradient-descent and ascent.}
\label{fig.saac_diag}
\end{figure}
\begin{algorithm}[!htb]
\caption{Smooth adversarial Actor-Critic (SaAC)}
\label{alg:saac}
\begin{algorithmic}
\STATE Initialize $\psi$, $\theta$, and $\phi$
\STATE Initialize learning rates $\beta_{\psi}, \beta_\theta$ and $\beta_{\phi}$
\STATE Initialize main iteration step $k=0$
\STATE Initialize adversarial policy updating interval $M$
\STATE Initialize buffer $\mathcal{B}\leftarrow\emptyset$
\REPEAT
\STATE // Sampling
    \FOR{each environment episode}
        \STATE $t=0$
        \STATE Randomly select an initial state $s$
        \FOR{each environment step}
            \STATE Sample adversarial action $u_t \sim\mu_\phi(\cdot|s_t)$
            \STATE Sample protagonist action $a_t\sim~\pi_{\theta}(\cdot|s_t)$
            \STATE Obtain $s_{t+1}\sim p$ and $r_t$
            \STATE Add the sample into buffer $\mathcal{B}\cup\{s_t, a_t, r_t, s_{t+1}\}$
            \STATE $t=t+1$
        \ENDFOR
    \ENDFOR
\STATE // Optimizing
\STATE Fetch a batch of samples from $\mathcal{B}$, compute $y_{\overline{\psi}}$ with Algorithm \ref{alg:target_value}, compute $J_{\pi,\mu}$ with $\mu_\phi$, $\pi_\theta$ and $p$
\STATE Update the parameters of value function \\
\qquad \qquad $w  \leftarrow w  - \beta_{w}\nabla_{w }J_{V}$
\STATE Update the parameters of policy \\
\qquad \qquad $\theta  \leftarrow \theta  - \beta_{\theta}\partial_{\theta} J_{\pi,\mu}$
\IF{$k \% M = 0$}
\STATE Update the parameters of adversary policy \\
\qquad \quad $\phi  \leftarrow \phi + \beta_{\phi}\partial_{\phi}J_{\pi,\mu}$
\ENDIF
\STATE $k=k+1$
\UNTIL Convergence
\end{algorithmic}
\end{algorithm}
\section{Simulation verification}
\label{sec.simulation_verification}
In this section, we carefully design two different tasks, the two-state MGs and the robust path tracking, which respectively corresponds to the tabular and function approximation settings, to verify the approximation performance of SPI and the training stability of SaAC.

\subsection{Two-state MGs}
\label{sec:simu_two-state MGs}
We demonstrate the approximation performance of SPI by introducing a two-state zero-sum MGs, which was originally proposed as a classical counterexample to illustrate the divergence of NPI \cite{van1978discounted}. With discrete and countable state and action spaces, SPI is able to solve the value and policies exactly in a state-by-state manner, making it unnecessary to employ any function approximator.
As shown in Fig.~\ref{fig.two-state MGs}, there are only two states in the state space $\mathcal{S}=\{s_1,s_2\}$, and $s_2$ is an absorbing state, that is, the agent will always keep motionless at $s_2$ no matter what actions are adopted. Moreover, there are only two corresponding elements in the protagonist action and the adversarial action space, i.e., $\mathcal{A}=\{a_1,a_2\}$ and $\mathcal{U}=\{u_1,u_2\}$. If the current state is $s_1$, the agent stays at $s_1$ regardless of the action pair in $\left\{(a_1,u_1),(a_2,u_1),(a_2,u_2)\right\}$, and receives rewards -3, -2, and -1 respectively. However, when taking the action pair $(a_1,u_2)$, the agent either stays at $s_1$ with probability $\frac{1}{3}$ or transits to $s_2$ with probability $\frac{2}{3}$ and receives reward -6 in both situations.
\begin{figure}[!htbp]
\centering
\captionsetup[subfigure]{justification=centering}
\includegraphics[width=0.85\linewidth]{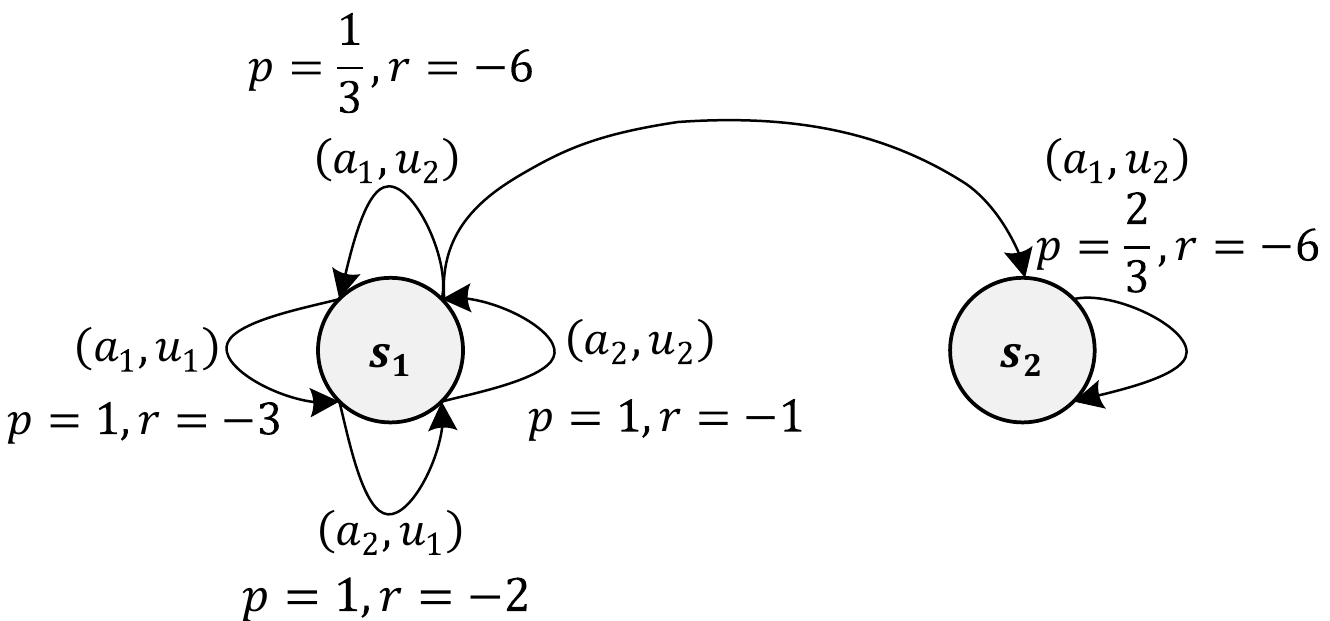}
\caption{Two-state zero-sum MGs.}
\label{fig.two-state MGs}
\end{figure}

Since $s_2$ is an absorbing state with zero cost, its state value $v^{\pi}(s_2)$ is always zero, and thus we only focus on the value function at $s_1$. 
Here, we firstly explore the difference of API and SPI, wherein their PEV should be accomplished iteratively according to \eqref{eq.api_pev} and \eqref{eq.spi_pev} given a pair of policies $\pi$ and $\mu$.
Specifically, we consider $\rho\in\{1.0,5.0,10.0,20.0\}$ respectively to further investigate the role of $\rho$ on the approximation performance.
Besides, to verify the influence of weights, the SPI-u algorithm is introduced under $\rho=10$ wherein a uniform distribution serves as the weight function regardless of the adversarial policy.
At the beginning, we initialize both the protagonist and adversarial policy as the stochastic policies, i.e., $$\pi_0(a_1|s_1)=0.50, \pi_0(a_2|s_1)=0.50,$$
$$\mu_0(u_1|s_1)=0.45, \mu_0(u_2|s_1)=0.55.$$
For this policy pair $(\pi_0, \mu_0)$, Fig.~\ref{fig.two-state mg pev 1st round} depicts the value varying with the iterative step during PEV. We can see that all methods basically converge to their fixed points when the step reaches more than 10, and SPI-based methods indeed obtain different approximation value depend on $\rho$ and weight. Furthermore, we list their convergent values in Table \ref{tab.appr_error_1}, wherein API converges to the true value of current policies for $s_1$, i.e., $v^{\pi_0}(s_1)=-6.9999$.
Obviously, the approximation of SPI becomes more accurate as $\rho$ increases, which is consistent with our previous theoretical analysis. SPI estimates the value as $-7.0693$ with $\rho=20.0$, where the approximation error to the value estimate from API is $\Delta e=0.98\%$.
For the same $\rho=10$, SPI achieves a better performance than SPI-u indicating the adversarial policy as the weight function will reduce the error bound efficiently than a random uniform distribution.
\begin{figure}[!htbp]
\centering
\captionsetup[subfigure]{justification=centering}
\subfloat[PEV for $\pi_0(s_1)$]{\label{fig.two-state mg pev 1st round}\includegraphics[width=0.8\linewidth]{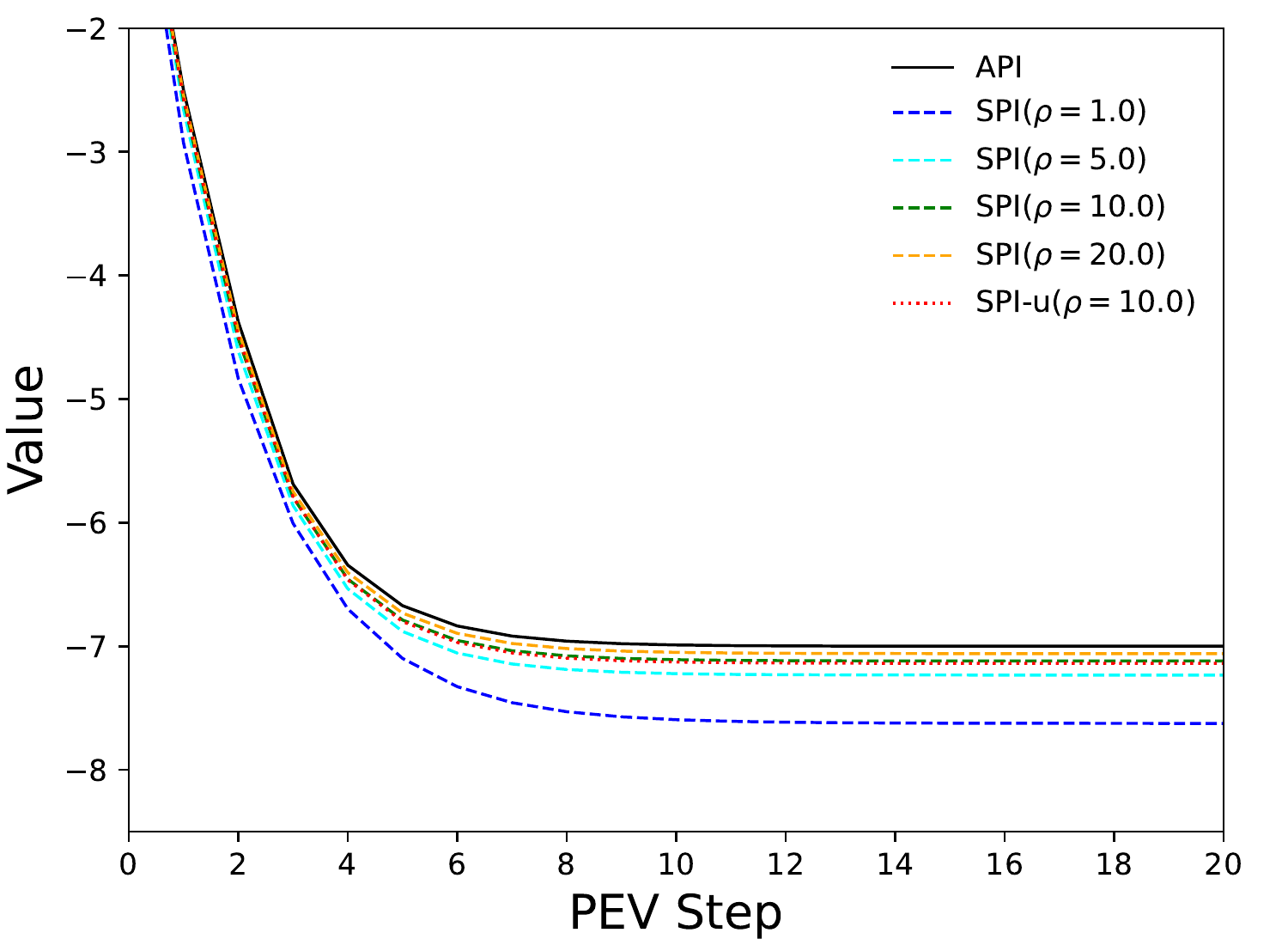}} \\
\subfloat[PEV for $\pi_1(s_1)$]{\label{fig.two-state mg pev 2nd round}\includegraphics[width=0.8\linewidth]{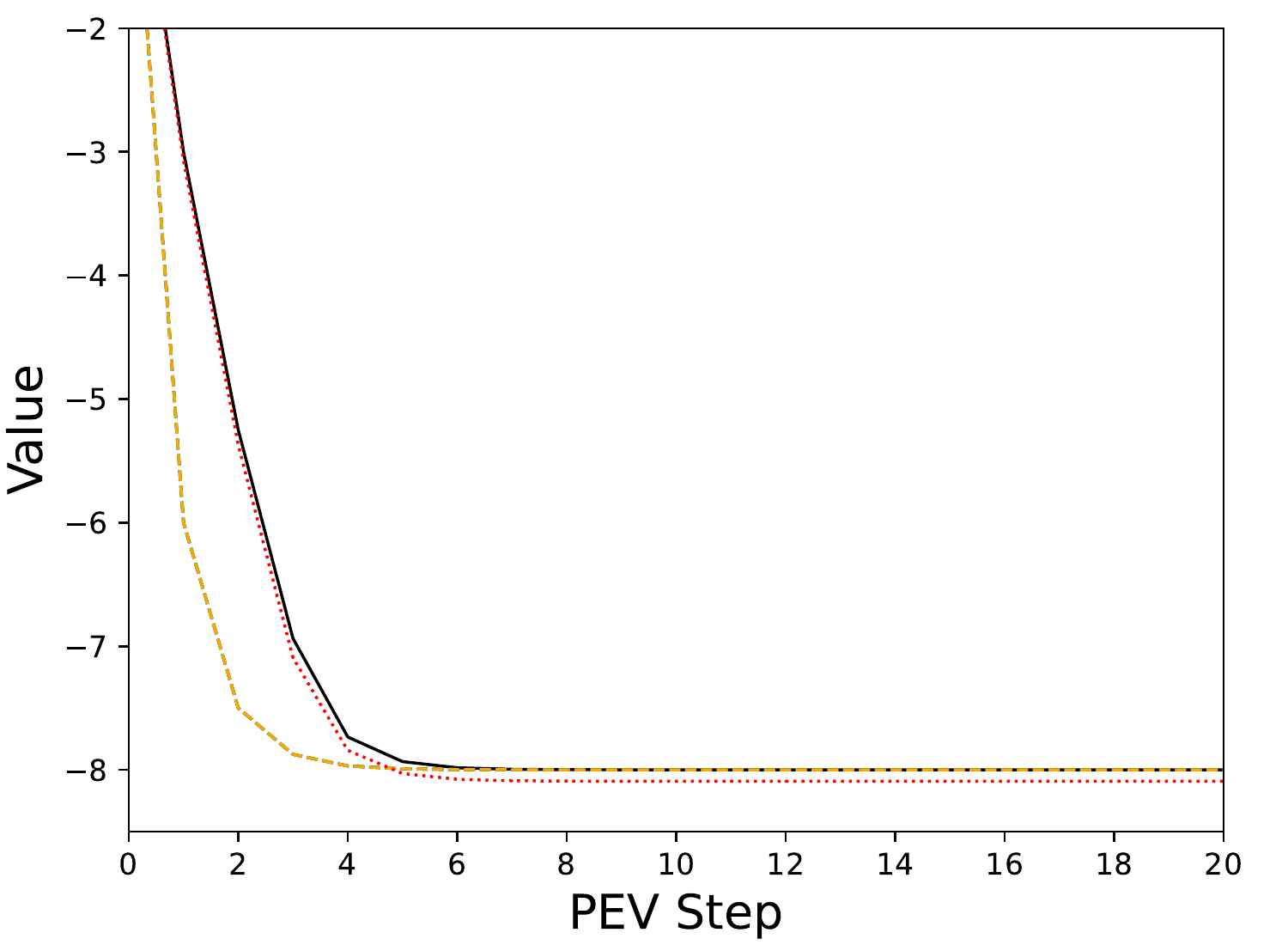}}
\caption{Process of policy evaluation.}
\label{fig.two-state mg pev}
\end{figure}

\begin{table}[tbhp]
\centering
\caption{Approximation error of for $\pi_0(s_1)$.}
\label{tab.appr_error_1}
\begin{tabular}{llll}
\toprule
Method & $\rho$ & Value & Error(\%) \\
\midrule
\multirow{4}{*}{SPI}&1.0 & -7.6243 & 8.92  \\
~&5.0 & -7.2334 & 3.34 \\
~&10.0 & -7.1195 & 1.71 \\
~&20.0 & -7.0598 & 0.86 \\
\hline
SPI-u&10.0 & -7.1385 & 1.98 \\
\hline
API&- & -6.9999 & 0.00 \\
\bottomrule
\end{tabular}
\end{table}

Next, with the converged value estimate $v^{\pi_0}$, we can conduct the PIM based on \eqref{eq.spi_pim} by solving the corresponding game matrix, i.e.,
\begin{equation}
\nonumber
\bordermatrix{
       & u_1       & u_2\cr
a_1    & -8.25         &-7.75  \cr
a_2    & -7.25         &-6.25  \cr }
\end{equation}
Clearly, the NE of this game matrix is $(a_1,u_2)$ with the equilibrium value as $-7.75$. Hence, the next policy pair $(\pi_1, \mu_1)$ can be extracted as
$$\pi_1(a_1|s_1)=1.0, \mu_1(u_2|s_1)=1.0,$$
and both the policies have degenerated into the deterministic form at this time.
Then, the next round of PEV will be conducted to evaluate $(\pi_1, \mu_1)$ as shown in Fig.~\ref{fig.two-state mg pev 2nd round}, where the true value calculated by API is $v^{\pi_1}(s_1)=-8.00$. 
Meanwhile, SPI-based method will always estimate $v^{\pi_1}(s_1)$ as the $-7.99$ regardless of the value of $\rho$, because the weight function $\mu_1$ only makes a difference at $u_2$. At this moment, their approximation errors are almost 0. However, SPI-u method still
maintains the considerable error about 1.12\%, meaning that the uniform weight function has limited ability to sample the most adversarial action.
\begin{table}[tbhp]
\centering
\caption{Approximation error of for $\pi_1(s_1)$.}
\label{tab.appr_error_2}
\begin{tabular}{llll}
\toprule
Method & $\rho$ & Value & Error(\%) \\
\midrule
\multirow{4}{*}{SPI}&1.0 & -8.00 & 0.00  \\
~&5.0 & -8.00 & 0.00 \\
~&10.0 & -8.00 & 0.00 \\
~&20.0 & -8.00 & 0.00 \\
\hline
SPI-u&10.0 & -8.09 & 1.12\ \\
\hline
API&- & -8.00 & 0.00 \\
\bottomrule
\end{tabular}
\end{table}
Subsequently, PIM aims to find the next policy pair by solving the following game matrix, i.e.,
\begin{equation}
\nonumber
\bordermatrix{
       & u_1       & u_2\cr
a_1    & -9.00 & -8.00  \cr
a_2    & -8.00 & -7.00  \cr }
\end{equation}
Since the NE of this matrix is still $(a_1,u_2)$, the extracted new policy $(\pi_2,\mu_2)$ is exactly identical to $(\pi_1,\mu_1)$, and thus their PEV will obtain the same $v^{\pi_2}(s_1)$ with $v^{\pi_1}(s_1)$, i.e., $v^{\pi_1}(s_1)=v^{\pi_2}(s_1)$. Therefore, the algorithm converges and we can conclude that the optimal NE policies ca be 
\begin{equation}
\label{eq:opt_ne_p}
\begin{aligned}
\pi^*(a_1|s_1)=1.0,\quad &\mu^*(u_2|s_1)=1.0
\end{aligned}
\end{equation}
with the equilibrium value as
\begin{equation}
\label{eq:opt_ne_v}
\begin{aligned}
v^*(s_1)=1.0,\quad &v^*(s_2)=0.0.
\end{aligned}
\end{equation}

The above analysis illustrates the convergence of API and SPI, and the results suggest that SPI can accurately approximate API by utilizing sufficiently large $\rho$ as well as a reasonable weight function. 
Moreover, as for NPI, we briefly illustrate the fact that value function of different policy pairs oscillates under some mild initial conditions, resulting in non-convergence at this two-state zero-sum MGs. 
To showcase this oscillation, the initial policy pair $(\pi_0, \mu_0)$ can be chosen as the deterministic form
\begin{equation}
\nonumber
\begin{aligned}
\pi_0(a_1|s_1)=1.0,
\mu_0(u_1|s_1)=1.0.
\end{aligned}
\end{equation}
Using the PEV process in \eqref{eq.npi_pev}, the joint value function esitimated by NPI is $v^{\pi_0,\mu_0}(s_1)=-12.00$, and the corresponding game matrix for PIM is
\begin{equation}
\nonumber
\bordermatrix{
       & u_1       & u_2\cr
a_1    & -12.00 & -9.00  \cr
a_2    & -11.00 & -10.00  \cr }
\end{equation}
Therefore, we can obtain the next policy pair $(\pi_1, \mu_1)$ as
\begin{equation}
\nonumber
\begin{aligned}
\pi_1(a_2|s_1)=1.0, \mu_1(u_2|s_1)=1.0.
\end{aligned}
\end{equation}
Similarly, its joint value function can be evaluated by PEV as
$v^{\pi_1,\mu_1}=-4.00$ and the corresponding game matrix is
\begin{equation}
\nonumber
\bordermatrix{
       & u_1       & u_2\cr
a_1    & -6.00 & -7.00 \cr
a_2    & -5.00 & -4.00  \cr }
\end{equation}
Consequently, the new policy pair $(\pi_2,\mu_2)$ extracted by PIM
become 
\begin{equation}
\nonumber
\begin{aligned}
\pi_1(a_1|s_1)=1.0, \mu_1(u_1|s_1)=1.0.
\end{aligned}
\end{equation}
which is identical to $(\pi_0,\mu_0)$ and we further conclude $v^{\pi_0,\pi_0}(s_1)=v^{\pi_2,\pi_2}(s_1)=-12.0$ and $v^{\pi_1,\pi_1}(s_1)=v^{\pi_3,\pi_3}(s_1)=-4.0$.
Actually, if we repeat the above PEV and PIM with more iteration, the value function will always oscillate between $-12.0$ and $-4.0$, unable to reach the optimal policies in \eqref{eq:opt_ne_p} and the optimal value in \eqref{eq:opt_ne_v}.
That is, NPI will diverge under this initial conditions and fail to find the optimal solutions.

\subsection{Robust Path Tracking}
Next, we show another classical robust path-tracking problem with continuous state and action spaces to characterize the training stability of SaAC. In this task, the vehicle aims to track a given path as accurately as possible in the presence of lateral perturbation as shown in Fig.~\ref{fig.path tracking problem}.
\begin{figure}[htbp]
\centering
\captionsetup[subfigure]{justification=centering}
\includegraphics[width=0.82\linewidth]{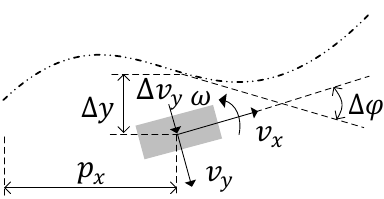}
\caption{Robust path tracking task. There exists an unknown disturbance in the lateral velocity.}
\label{fig.path tracking problem}
\end{figure}
The vehicle state is designed as $s=\left[p_x,\Delta y,\Delta \varphi, v_x, v_y, \omega\right]^\top$, where the elements are longitudinal position, lateral tracking error, heading angle error, longitudinal velocity, lateral velocity and raw rate respectively. 
Besides, we choose the front wheel angle $\delta$ and acceleration $A$ as the protagonist action to realize the lateral and longitudinal control, i.e., $a=[\delta, A]^\top$. At the same time, there exists a gaussian uncertainty $\Delta v_y \sim \mathcal{N}(\kappa,\Lambda^2)$ in the lateral velocity $v_y$ due to crosswinds or road lateral slopes, wherein $\kappa$ and $\Lambda$ are the learnable mean and standard deviation. Hence, the adversarial policy is expected to learn the worst $\Delta v_y$ that interferes with the tracking performance of the ego vehicle, i.e., $u=\Delta v_y$. Considering the actuator saturation, we bound the protagonist action and adversarial action with $\delta \in [-0.4, 0.4] {\rm rad}$, $A \in [-1.5, 3.0] {\rm m/s^2}$ and $\Delta v_y \in [-0.5, 0.5] {\rm m/s}$, respectively.  
The dynamics of this task are modeled as follows based on \cite{guan2021mixed}:
\begin{equation}
\label{eq.vehicle dynamical model}
p({s}, {a}, u)=\left[\begin{array}{c}
p_x+\Delta t\left(v_x \cos \Delta \varphi-v_y \sin \Delta \varphi\right) \\
\Delta y+\Delta t\left(v_x \sin \Delta \varphi+v_y \cos \Delta \varphi\right) \\
\Delta \varphi+\Delta t \omega \\
v_x+\Delta t\left(A+v_y \omega\right) \\
\frac{m v_x v_y+\Delta t\left[\left(L_f k_f-L_r k_r\right) \omega-k_f \delta v_x-m v_x^2 \omega\right]}{m v_x-\Delta t\left(k_f+k_r\right)}+u \\
\frac{-I_z \omega v_x-\Delta t\left[\left(L_f k_f-L_r k_r\right) v_y-L_f k_f \delta v_x\right]}{\Delta t\left(L_f^2 k_f+L_r^2 k_r\right)-I_z v_x}
\end{array}\right],
\end{equation}
where the dynamical parameters are shown in Table \ref{tab.dynamical parameters of robust path tracking}. 


\begin{table}[tbhp]
\centering
\caption{Dynamical parameters.}
\label{tab.dynamical parameters of robust path tracking}
\begin{tabular}{lll}
\toprule
Parameter & Symbol & Value \\
\midrule
Front-wheel cornering stiffness & $k_f$ & -155495N/rad  \\
Rear-wheel cornering stiffness & $k_r$ & -155495N/rad \\
Distance from CG to front axle & $L_f$ & 1.19m \\
Distance from CG to rear axle & $L_r$ & 1.46m \\
Mass & $m$ & 1520kg \\
Polar moment of inertia at CG & $I_z$ & 2640kg$\cdot\mathrm{m}^2$ \\
Discrete time step & $\Delta t$ & 0.1s\\ 
\bottomrule
\end{tabular}
\end{table}

Considering the tracking accuracy, energy efficiency and driving comfort, we construct a classical quadratic-form reward function with carefully tuned weights as
\begin{equation}
\nonumber
\label{eq.utility function of robust path tracking}
\begin{aligned}
    &r(s,a,u)= 0.03 (v_x -20)^{2}+0.8 {\Delta y}^{2}+30 {\Delta \varphi}^{2} +0.05 A^{2}\\
    &+0.02 \omega^{2}+5 \delta^{2},
\end{aligned}
\end{equation}
in which the objective is to maintain the vehicle speed close to 20m/s, while keeping a small tracking error and ensuring that the vehicle stays within the stability region. 
Then, $\pi_{\theta}$ is to minimize the accumulated rewards, while $\mu_{\phi}$ aims to maximize them to exacerbate the driving
performance by outputting the lateral interference. 
Thus, the protagonist policy trained in this setting takes into account possible disturbances during training and should become more robust against the distractions from real environment.

The reference path is the composition of three sine curves with different magnitudes and periods,
\begin{equation}
\nonumber
\label{eq.ref_path}
y_{\rm ref}=7.5\sin\frac{2\pi}{200} +2.5\sin\frac{2\pi}{300}-5\sin\frac{2\pi}{400}.
\end{equation}
To verify the effectiveness of SaAC, we select the model-based version of robust adversarial reinforcement learning (RaRL) \cite{pinto2017robust} as a comparison benchmark, which is indeed a parameterized version of NPI. The only difference of SaAC and RaRL is the target value in value network update, wherein the former uses the bellman operator to construct the target and the latter adopts the smooth bellman operator to calculate it.
We also adopt another classic model-based RL method called approximate dynamic programming (ADP) wherein only the protagonist policy is constructed without the adversarial training, to showcase the robust advantages of SaAC.
Finally, we also construct the SaAC-u algorithm, the extension of SPI-u with neural networks, where the continuous uniform distribution is used as the weight function rather than the adversarial policy.
For fair comparison, all these algorithms adopt the same algorithmic framework and training parameters. We utilize multi-layer perceptron (MLP) as the function approximator of value function, protagonist policy and adversarial policy. Each network contains five 256-unit hidden layers and we apply the Gaussian error linear unit (GeLU) as their activation functions. In addition, hyperbolic tangent (tanh) is utilized for the output layers of policy networks to saturate their outputs. Meanwhile, the value network uses a linear function as the activation of its output layer. With a cosine-annealing learning rate, Adam \cite{kingma2014adam} is used to update all networks. We also deploy an asynchronous parallel training framework \cite{guan2021mixed} to accelerate the training process. Specifically, 2 samplers pick up data by interacting with the environment, 2 buffers store samples, and 10 learners optimize the networks by exploiting the collected data. The detailed hyperparameter settings are listed in Table \ref{table.hyper of robust path tracking}.
\begin{table}[!htbp] 
\centering
\caption{Training hyperparameters.}
\label{table.hyper of robust path tracking}
\begin{threeparttable}[h]
\begin{tabular}{ll}
\toprule
Hyperparameter & Value \\
\hline
Optimizer &  Adam ($\beta_{1}=0.9, \beta_{2}=0.999$)\\
Function approximator  & MLP \\
Number of hidden layers & 5\\
Number of hidden units & 256\\
Activation of hidden layer & GeLU\\
Activation of output layer & Tanh (policy) / Linear (Value)\\
Batch size & 256\\
Learning rate anneal & Cosine anneal \\
Policy learning rate & $5\times10^{-5}$ $\rightarrow$ $1\times10^{-6}$ \\
Value learning rate & $8\times10^{-5}$ $\rightarrow$ $1\times10^{-6}$\\
Discount factor & 0.99 \\
Policy updating interval $M$ & 1 \\
Temperature of target network & 0.001 \\
Number of samplers & 2 \\
Number of buffers & 2 \\
Number of learners & 10\\
\bottomrule
\end{tabular}
\end{threeparttable}
\end{table}

During training, we test the tracking performance of the learned protagonist policy every 3,000 iterations. In each test, the ego vehicle will be initialized randomly and attempts to track the given path for 150 steps. Then, we will calculate the total average return of 5 episodes under the same policy to show its current performance.
The training results for 5 runs with different seeds are shown in Fig.~\ref{fig.test_total_return}.
In Table \ref{tab.test_return}, we also show the converged return and characterize the convergence speed of different algorithms, which is measured by the average number of iterations needed to firstly exceed a certain goal performance, set -25 in our experiments. 
Clearly, SaAC obtains the highest return when converged which is around -5 and SaAC-u almost has the same return but it suffers a low convergence speed, needing 54000 iterations to acquire a fair performance.
This also demonstrates that the adversarial policy will be a better weight function than a random distribution, which will accelerates the policy evaluation and thus converges rapidly.
Besides, ADP finally attains a lower return than SaAC and SaAC-u, which means that it cannot deal with the disturbance from environment perfectly with the normal training process.
Finally, we can see RaRL achieves the lowest average return over the 5 runs, and Table \ref{tab.test_return} shows that it suffers a rather large variance in the converged return under different seeds.
We can conclude that RaRL is very unstable in training under different initial conditions, indicating the similar properties with SPI in tabular case.


\begin{figure}[htbp]
\centering
\captionsetup[subfigure]{justification=centering}
\includegraphics[width=0.95\linewidth]{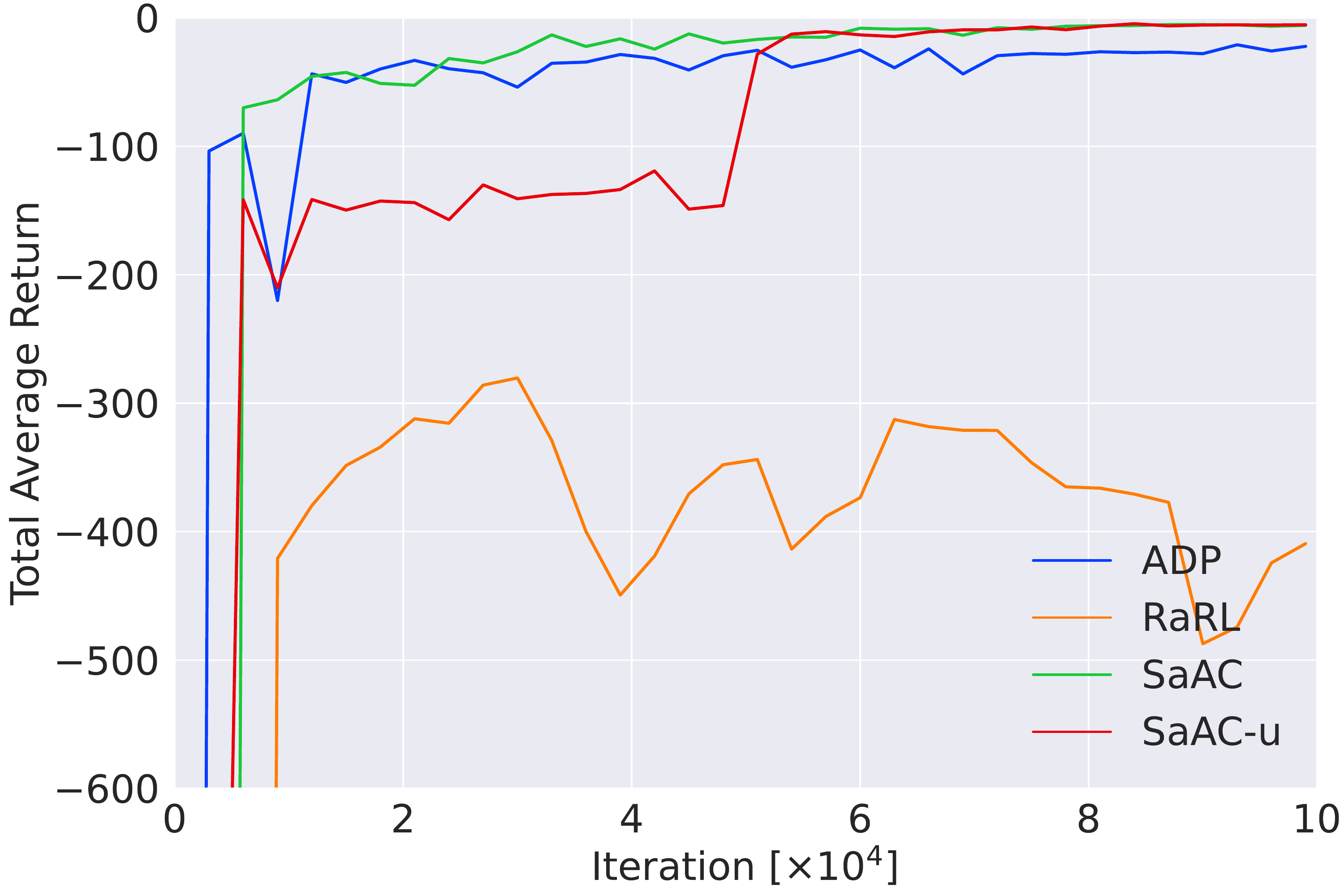}
\caption{Learning curves. The solid lines correspond to the mean over 10 runs.}
\label{fig.test_total_return}
\end{figure}
\begin{table}[tbhp]
\centering
\caption{TAR of each run.}
\label{tab.test_return}
\begin{tabular}{lll}
\toprule
Algorithm & Performance & Convergence speed \\
\midrule
ADP & $-20.92 \pm 12.40 $ & 60000 \\
RaRL & $-280.35 \pm 379.64 $ & - \\
SaAC & \bm{$-5.18 \pm 0.69$} & \textbf{33000} \\
SaAC-u & $-5.38 \pm 0.61$ & 54000 \\
\bottomrule
\end{tabular}
\end{table}

Fig.~\ref{fig.track_error} graphs the tracking errors of all baseline algorithms
during training. We can see RaRL hardly tracks the reference path accurately because there exist high errors in position. Besides, SaAC-u shows a slow convergence speed in heading error compared with SaAC. Finally, ADP obtains a higher position error and heading angle error than SaAC, meaning the disturbance will work obviously if the training process is conducted without considering the environmental variation.
\begin{figure}[htbp]
\centering
\captionsetup[subfigure]{justification=centering}
\subfloat[Position error]
{\label{fig.pos_error}
\includegraphics[width=0.90\linewidth]{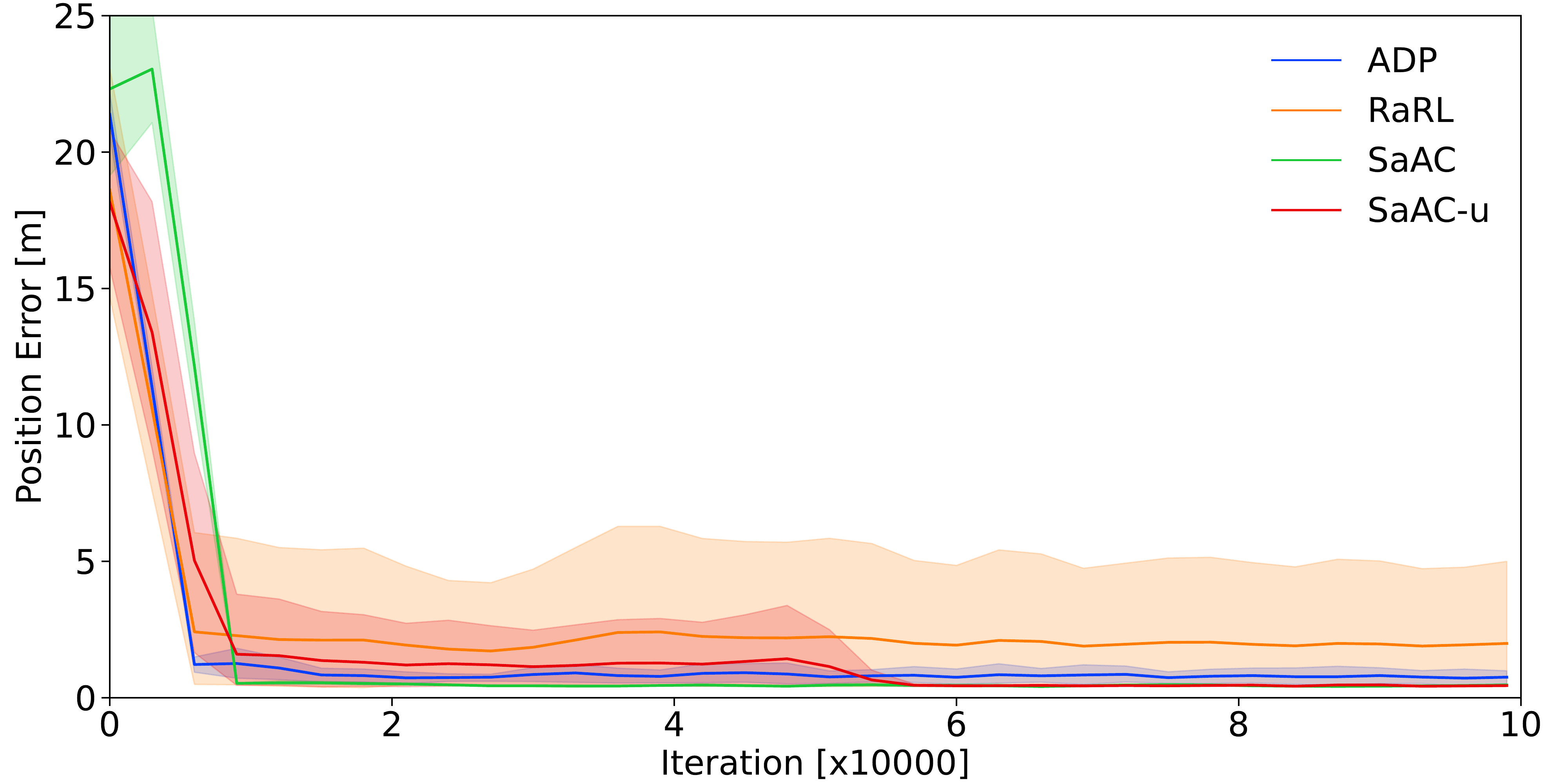}} \\
\subfloat[Heading angle error]
{\label{fig.head_error}
\includegraphics[width=0.90\linewidth]{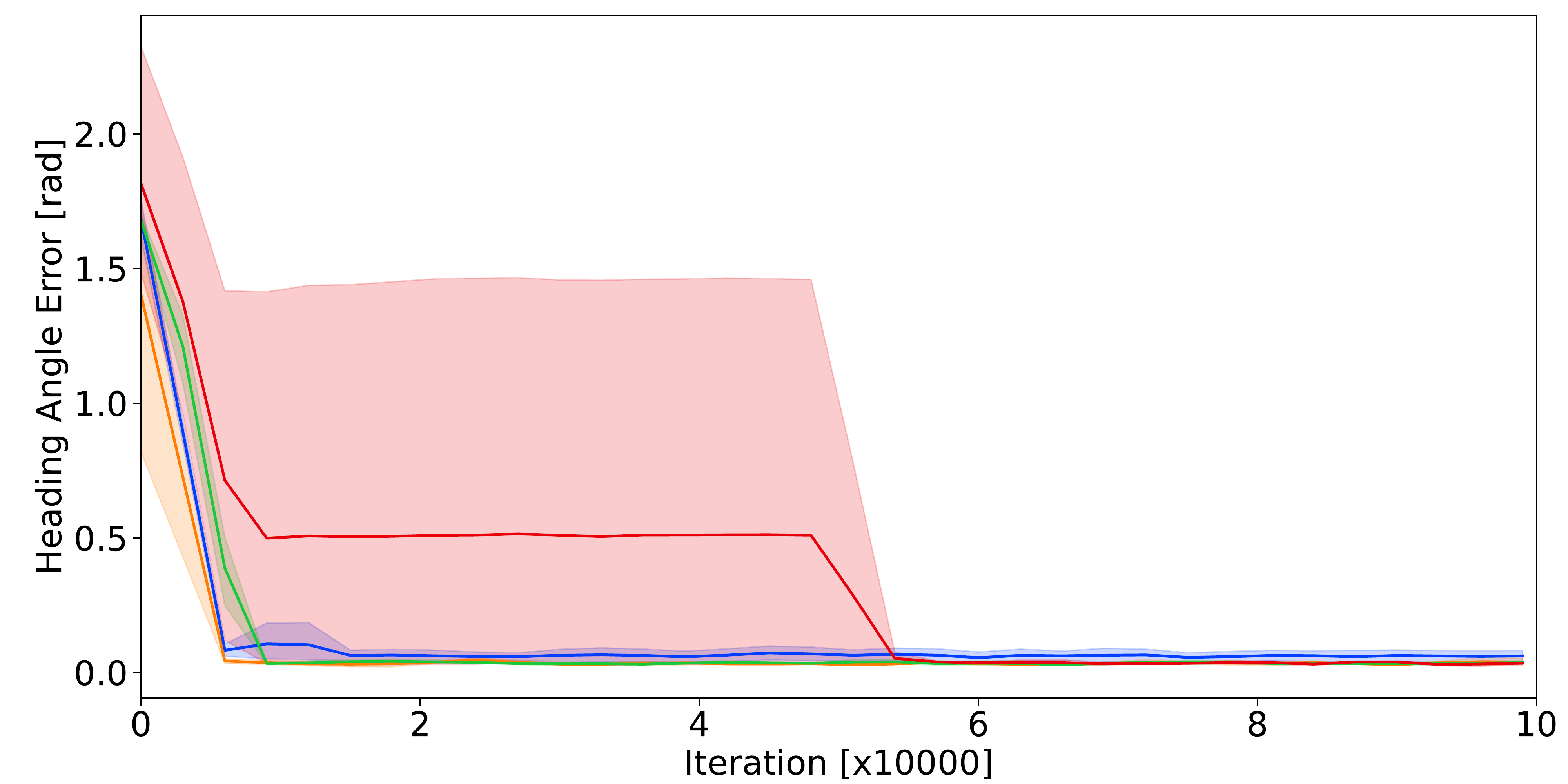}}
\caption{Errors of robust path tracking task. (a) Position error. (b) Heading error. The solid lines correspond to the mean and the shaded regions correspond to 95\% confidence interval over 5 runs.}
\label{fig.track_error}
\end{figure}
We also conduct the robust test of the trained policy in Fig.~\ref{fig.robust_test}, wherein the protagonist policies of ADP and SaAC after 100000 iterations under the same random seed are chosen to track the path for 5 episodes.
Meanwhile, we establish a varying disturbance to the lateral velocity of the model in \eqref{eq.vehicle dynamical model} from $-0.3m/s$ to $0.3m/s$ with the interval 0.06.
Obviously, the total average return respectively reaches their highest point near the disturbance of 0 and the bigger disturbance will cause the performance degradation for these two algorithms. However, SaAC always behaves better than ADP under all these disturbances, which is befit from the robustness of adversarial training. 

To sum up, SPI can acquire a satisfactory solution of adversarial Bellman equation with the large factor $\rho$ and the adversarial weight function. And with neural network as function approximators, SaAC can maintain a fair training stability as well as improve the robustness of protagonist policy in a large margin.

\begin{figure}[htbp]
\centering
\captionsetup[subfigure]{justification=centering}
\includegraphics[width=0.90\linewidth]{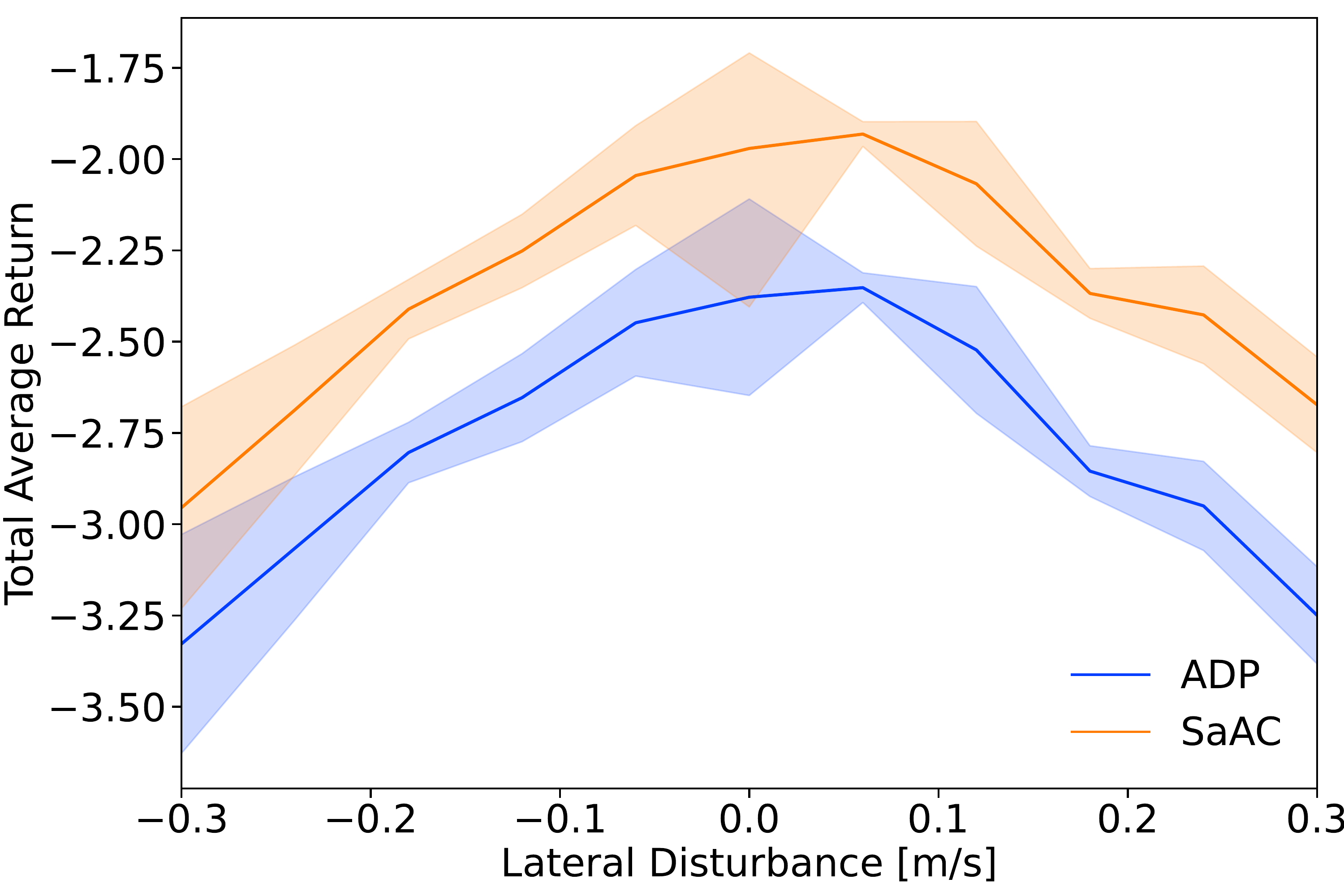}
\caption{Testing curves under lateral disturbance. The solid lines correspond to the mean and the shaded regions correspond to 90\% confidence interval over 5 episode.}
\label{fig.robust_test}
\end{figure}
\section{Conclusion}
This paper aims to solve the zero-sum MGs approximately to handle the complex tasks with large-scale action spaces. To this end, we propose the SPI algorithm wherein WLSE function is developed to approximate the maximum operator. Specially, we can control the approximation error by adjusting the approximation factor and the weighted function enables an efficient sampling in action spaces. We also prove the convergence of SPI and analyze its approximation error in $\infty -$norm based on the contraction mapping theorem.
Based on this, we propose SaAC by extending SPI with the neural networks as the function approximators, which is a typical model-based RL algorithm to train the protagonist and adversarial policy simultaneously.
We apply our algorithm in a two-state MGs and the robust path tracking tasks respectively. Results show that SPI can approximate the worst-case value function with a high accuracy and SaAC can stabilize the training process and lead to the adversarial robustness.
About the future work, we will extend the SPI-based algorithm to handle more complex industrial tasks such autonomous driving, where the self-driving car needs to compete with its surrounding participants to strive for a successful navigation. Moreover, we will consider the constraints to the adversarial policy in current minimax formation to make the less radical adversarial behaviors.

\bibliographystyle{IEEEtran}
\bibliography{cite}

\end{document}